\documentclass[aps,pra,reprint,superscriptaddress, floatfix]{revtex4-2}

\usepackage{graphicx}
\usepackage{subcaption}
\usepackage{amsmath, amsthm}
\usepackage{amssymb}
\usepackage{hyperref}
\usepackage{braket}
\usepackage{xcolor}
\usepackage{algorithm}
\usepackage{algpseudocode}
\usepackage{placeins}
\newtheorem{theorem}{Theorem}
\newtheorem{proposition}{Proposition}

\newcommand{\bb}[1]{\mathbb{#1}}
\renewcommand{\cal}[1]{\mathcal{#1}}
\renewcommand{\rm}[1]{\mathrm{#1}}
\newcommand{\E}[1]{\mathbb{E}\left[ #1 \right]}
\renewcommand{\O}[1]{\mathcal{O}\left(#1\right)}

\begin{document}

\title{High-Resolution Tensor-Network Fourier Methods for Exponentially Compressed Non-Gaussian Aggregate Distributions}

\author{Juan José Rodríguez-Aldavero}
\email{jj.aldavero@csic.es}
\affiliation{Institute of Fundamental Physics IFF-CSIC, Calle Serrano 113b, Madrid 28006, Spain}
\author{Juan José García-Ripoll}
\affiliation{Institute of Fundamental Physics IFF-CSIC, Calle Serrano 113b, Madrid 28006, Spain}

\date{\today}

\begin{abstract}
    Characteristic functions of weighted sums of independent random variables exhibit low-rank structure in the quantized tensor train (QTT) representation, also known as matrix product states (MPS), enabling up to exponential compression of their fully non-Gaussian probability distributions. Under variable independence, the global characteristic function factorizes into local terms. Its low-rank QTT structure arises from intrinsic spectral smoothness in continuous models, or from spectral energy concentration as the number of components $D$ grows in discrete models. We demonstrate this on weighted sums of Bernoulli and lognormal random variables. In the former, despite an adversarial, incompressible small-$D$ regime, the characteristic function undergoes a sharp bond-dimension collapse for $D \gtrsim 300$ components, enabling polylogarithmic time and memory scaling. In the latter, the approach reaches high-resolution discretizations of $N = 2^{30}$ frequency modes on standard hardware, far beyond the $N = 2^{24}$ ceiling of dense implementations. These compressed representations enable efficient computation of Value at Risk (VaR) and Expected Shortfall (ES), supporting applications in quantitative finance and beyond.
\end{abstract}

\maketitle

\section{Introduction}
\label{sec:introduction}

Weighted sums of independent random variables constitute a basic probabilistic model, describing macroscopic behavior arising from the aggregation of microscopic stochastic components. These models arise in a wide range of applications. For instance, weighted discrete sums underpin standard credit risk models, where portfolio losses depend on the joint survival of multiple obligors~\cite{vasicek1987, glasserman2005, bluhm2016}, and appear in multi-state reliability analysis, where aggregate capacity depends on heterogeneous components with varying performance levels~\cite{lisnianski2010, li2012}. In continuous settings, sums of lognormal variables describe the total value of variable-income portfolios~\cite{osborne1959}, enable pricing financial derivatives~\cite{dufresne2004}, and characterize the signal power in wireless communication channels~\cite{fenton1960, schwartz1982, mehta2007}.

While conceptually simple, these models present significant analytical and computational challenges. Their probability distribution generally lacks a closed-form expression, and their evaluation involves multidimensional convolution integrals that are susceptible to the curse of dimensionality. Although the central-limit theorem (CLT) guarantees asymptotic convergence to a Gaussian distribution~\cite{feller1991}, this convergence is often slow and non-uniform, making it insufficient in many practical settings, particularly when estimating tail probabilities for skewed or heavy-tailed distributions.

Consequently, evaluating these models relies on specialized numerical methods. Typically, Monte Carlo sampling is used~\cite{robert1999, glasserman2004}, as it circumvents the curse of dimensionality, although it converges slowly in low-probability regions and requires a prohibitively large number of samples to resolve tail behavior accurately. For discrete models, recursive convolution methods provide an exact evaluation~\cite{panjer1981, wang1993}, though their performance is limited by the size of the model's support and becomes inefficient in general scenarios. A more sophisticated alternative relies on Fourier inversion of the model's characteristic function, effectively performing the multidimensional convolution in the frequency domain~\cite{abate1992, carrmadan1999}. While these methods have been adapted for discrete settings~\cite{fang2009, hong2013}, they are frequently hampered by persistent Gibbs oscillations, which arise from distributional discontinuities and preclude uniform convergence~\cite{tadmor2007, ruijter2015}. No existing method simultaneously achieves an accurate approximation of the exact, fully non-Gaussian target distribution while remaining scalable to larger, practically relevant system sizes.

In this work, we introduce a new algorithm that combines the Fourier spectral method with tensor-network techniques. The method represents the global characteristic function as a quantized tensor train (QTT)~\cite{oseledets2011, khoromskij2011}---also known as a matrix product state (MPS)~\cite{perez2006, verstraete2008, orus2014, cirac2021}---and recovers the probability distribution via Fourier inversion using superfast Fourier transforms~\cite{dolgov2012, chen2023, chen2024}. By exploiting the factorization of the global characteristic function under variable independence, we combine the local characteristic functions of the independent components via sequential QTT Hadamard products~\cite{garcia2021, garcia2023, rodriguez2024}. For discrete random variables, these local characteristic functions admit exact low-rank QTT representations given by sums of complex exponentials. While continuous random variables often exhibit a similar algebraic construction, tensor cross-interpolation (TCI) provides a general fallback~\cite{oseledets2010,savostyanov2011,dolgov2020,nunez2022,ritter2024}. To mitigate Gibbs oscillations, we apply spectral filtering directly to the tensor-network representation, ensuring pointwise convergence away from the discontinuities and providing rigorous global convergence guarantees~\cite{vandeven1991, gottlieb1997}. Whenever low-rank structure is present, our QTT-based approach achieves polylogarithmic time and memory complexity in the grid size, scaling as $\O{\log \varepsilon^{-1}}$ in the worst case and $\O{\log \log \varepsilon^{-1}}$ in the best case for a prescribed global accuracy $\varepsilon$. Crucially, the method captures fully non-Gaussian target distributions directly, with efficiency deriving from low-rank spectral structure rather than any Gaussian approximation.

We illustrate the performance of the algorithm on two non-trivial representative models. First, we benchmark weighted Bernoulli sums, whose discrete support grows exponentially for incommensurate weights, complicating numerical treatment. In smaller instances, these models remain incompressible and highly adversarial. However, they become compressible at $D \gtrsim 300$ components---a scale matching real-world credit portfolios and reliability models---where the bond dimension of the QTT characteristic function undergoes a sharp CLT-driven collapse driven by the multiplicative suppression of high-frequency modes, reducing computational requirements by orders of magnitude. This compressibility transition, which has not been previously reported, occurs at system sizes directly relevant to practical applications. Second, we consider sums of lognormal random variables, whose local components exhibit strong compressibility due to intrinsic smoothness, enabling polylogarithmic complexity in the grid size and efficiently reaching $N = 2^{30}$ frequency modes on standard hardware, far beyond the $N = 2^{24}$ ceiling of dense implementations. Moreover, the resulting compressed distributions enable the efficient computation of the Value at Risk (VaR)~\cite{jorion2007, embrechts2013b} and Expected Shortfall (ES)~\cite{artzner1999, acerbi2002} risk metrics directly in tensor-network form, underscoring practical applications in fields such as quantitative finance.

This work is structured as follows. Sect.~\ref{sec:problem} formulates the problem, introducing weighted-sum models and their convergence to the normal distribution. Sect.~\ref{sec:spectral} presents the spectral method and its convergence guarantees. Sect.~\ref{sec:qtt} develops the tensor-network spectral algorithm. Sect.~\ref{sec:results} benchmarks the method for discrete and continuous weighted-sum models, comparing it against conventional approaches. Finally, Sect.~\ref{sec:conclusion} concludes with a discussion on the results and future directions.

\section{Problem statement}
\label{sec:problem}

Let $X$ be a random variable defined as the weighted sum of $D$ independent real-valued components $\{X_d\}_{d=1}^D$,
\begin{equation}
    X = \sum_{d=1}^D w_d X_d,
    \label{eq:X}
\end{equation}
combined through fixed weights $w_d \in \bb{R}$. The probability density function of $X$ can be expressed as the multidimensional convolution of the individual densities $f_{X_d}$, given by
\begin{equation}
    f_X(x) = \int_{\bb{R}^D} \!\!\! \delta\left(x - \sum_{d=1}^D w_d x_d \right) \! \prod_{d=1}^D f_{X_d}(x_d) \ \rm{d} x_1 \dots \rm{d} x_D.
    \label{eq:f_X_convolution}
\end{equation}
Evaluating Eq.~\eqref{eq:f_X_convolution} presents significant analytical and numerical challenges. Closed-form expressions for $f_X$ exist only when the component variables belong to a restricted class of distributions---such as the Gaussian or Gamma families---for which the sum remains within the same distributional class. From a numerical perspective, Eq.~\eqref{eq:f_X_convolution} involves a $D$-dimensional integral that is susceptible to the curse of dimensionality when approached with standard quadrature methods. Moreover, the Dirac delta term imposes a global linear constraint that couples all variables, which prevents the application of techniques based on separability.

Often, the quantity of interest is the cumulative distribution function (CDF), defined as
\begin{equation}
    F_X(x) := \int_{-\infty}^x f_X(t) \ \rm{d}t = (f_X * \Theta)(x),
    \label{eq:F_X}
\end{equation}
where the last equality denotes convolution with the Heaviside step function. The CDF naturally encodes the quantiles of the aggregate variable and threshold exceedance probabilities. In risk analysis, where $X$ typically denotes the loss of a portfolio, it underpins Value at Risk (VaR), defined as the $\alpha$-quantile $F_X^{-1}(\alpha)$~\cite{jorion2007, embrechts2013b}, and Expected Shortfall (ES), which involves tail expectations beyond this quantile~\cite{artzner1999, acerbi2002}. However, since the CDF inherits the high-dimensional, non-separable structure of the underlying density, its numerical evaluation remains similarly challenging.

Weighted-sum models of the form~\eqref{eq:X} satisfy the central-limit theorem (CLT) under mild regularity conditions: as the number of components $D$ increases, the normalized sum converges in distribution to a Gaussian random variable, largely independently of the microscopic details of the component variables. This rate of convergence can be quantified by the following classical result~\cite{feller1991}.
\begin{theorem}{Berry--Esseen theorem}\\
Let $X_1, \ldots, X_D$ be independent random variables with $\E{X_d}=0$, variances $\sigma_d^2$, and finite third absolute moments $\E{|X_d|^3}$.  
Denote $\sigma^2 = \sum_{d=1}^{D} \sigma_d^2$ and $\rho = \sum_{d=1}^{D} \E{|X_d|^3}$.  
Then, the CDF $F_X$ of the standardized sum $X = \tfrac{1}{\sigma}\sum_{d=1}^{D} X_d$ satisfies
\[
    \sup_{x \in \bb{R}} |F_X(x) - F_G(x)| \le C \frac{\rho}{\sigma^3},
\]
where $F_G$ denotes the CDF of the standard Gaussian distribution and $C$ is a universal constant.
\label{thm:berry-esseen}
\end{theorem}
Typically, both the total variance $\sigma^2$ and the sum of third absolute moments $\rho$ scale linearly with $D$. Consequently, the deviations from normality decay at the rate $\O{D^{-1/2}}$, with the universal constant bounded by $0.4097 \leq C \leq 0.5583$~\cite{shevtsova2013}.

Despite this guaranteed asymptotic convergence, Gaussian approximations are often insufficient for practical applications. While Theorem~\ref{thm:berry-esseen} controls the deviations of the CDF, it allows significant errors in the density, particularly in low-probability regions. Moreover, the $\O{D^{-1/2}}$ convergence rate can be too slow for scenarios involving a moderate number of components, where the resulting approximation errors may exceed practical tolerance levels. 

These challenges motivate the development of specialized numerical methods. However, conventional approaches such as Monte Carlo sampling~\cite{robert1999,glasserman2004} (see Appendix~\ref{appendix:conventional-montecarlo}) and recursive convolution schemes~\cite{panjer1981, wang1993} (see Appendix~\ref{appendix:conventional-recursive}) exhibit intrinsic limitations. On the one hand, Monte Carlo methods converge with a statistical error that decays as $\O{S^{-1/2}}$ with the number of samples $S$, requiring a prohibitively large number of samples to accurately resolve low-probability regions. On the other hand, recursive convolution methods provide an exact evaluation for discrete models, but scale with the model's support size and become inefficient in general scenarios with exponentially dense and irregular supports. For these reasons, Fourier spectral methods provide an effective alternative, exploiting separability to map the $D$-fold convolution in Eq.~\eqref{eq:f_X_convolution} into an efficient pointwise product in the Fourier domain.  

\section{Fourier spectral method}
\label{sec:spectral}

Fourier spectral methods evaluate the probability distribution of a random variable by inverting its characteristic function. For the weighted-sum models defined in Eq.~\eqref{eq:X}, this approach effectively computes the $D$-fold convolution~\eqref{eq:f_X_convolution} directly in the frequency domain, where the convolution theorem transforms the high-dimensional integral into a product of one-dimensional components.

We briefly outline the main elements of this approach, which will serve as the foundation for the tensor-network formulation developed in Section~\ref{sec:qtt}.

\subsection{Characteristic function}
\label{sec:spectral-cf}

The characteristic function of a real-valued random variable $X$ is the Fourier transform of its probability density function, given by the expectation
\begin{equation}
    \phi_X(\omega) := \E{e^{i \omega X}},
    \label{eq:phi_X}
\end{equation}
where $\omega \in \mathbb{R}$ denotes the frequency variable.

Crucially, the characteristic function for weighted-sum models~\eqref{eq:X} admits the factorized representation
\begin{equation}
    \phi_X(\omega) = \prod_{d=1}^D \phi_{X_d}(w_d \omega),
    \label{eq:cf_factorization}
\end{equation}
where $\phi_{X_d}$ denotes the local characteristic function of the $d$-th variable, each of which is often available in closed form for classical parametric distributions.

The characteristic function uniquely determines both the probability density and the cumulative distribution function through inversion formulas such as the Gil--Peláez formula~\cite{gil1951}. Moreover, by Lévy's continuity theorem, this correspondence is continuous: convergence in distribution implies pointwise convergence of the associated characteristic functions~\cite{feller1991}. Consequently, in accordance with the central limit theorem, the global characteristic function of weighted-sum models converges as $D \to \infty$ to that of the Gaussian distribution,
\begin{equation}
    \phi_G(\omega) = \exp \left(i\mu\omega - \tfrac{1}{2}\sigma^2\omega^2\right),
    \label{eq:phi_X_normal}
\end{equation}
suggesting an increasingly smooth structure that can be exploited by tensor-network representations to achieve compressibility and computational efficiency~\cite{jobst2024}. Importantly, this low-rank structure manifests even in regimes where the underlying target distribution remains substantially non-Gaussian.

In numerical implementations, the frequency is discretized on a finite grid $\{\omega_k\}_{k=0}^{N-1} \subset [-\Omega, \Omega]$. The cutoff frequency $\Omega$ and the number of frequency modes $N$ jointly determine the accuracy and computational cost of the spectral reconstruction. As discussed in Appendix~\ref{appendix:frequency_discretization}, these parameters are related through the Shannon--Nyquist sampling theorem: resolving spectral content up to frequency $\Omega$ without aliasing requires $N \propto \Omega$ samples~\cite{shannon2006}. Consequently, the effective bandwidth of the distribution directly governs the computational complexity of its spectral reconstruction.

\subsection{Cumulative distribution function}
\label{sec:spectral-cdf}

The CDF can be obtained from the density through convolution with the Heaviside step function (see Eq.~\eqref{eq:F_X}). However, it is more efficient to evaluate it directly in Fourier space by applying the convolution theorem,
\begin{equation}
    F_X(x) = \cal{F}^{-1}\left( \phi_X(\omega) \cdot \cal{D}(\omega) \right),
\end{equation}
where $\cal{D} := \cal{F}(\Theta)$ denotes the Fourier transform of the Heaviside step function. Although this transform does not converge in the continuum and is only defined formally as a tempered distribution, numerical implementations on discrete frequency samples yield a well-defined representation known as the Dirichlet kernel~\cite{oppenheim1999}. 

A technical subtlety arises: standard implementations of the discrete Fourier transform (DFT) assume periodic boundary conditions and compute circular (wrapped-around) convolutions, whereas Eq.~\eqref{eq:F_X} requires a linear convolution. To avoid wrap-around artifacts, a zero-padding trick is used: the density is embedded into a larger grid by effectively doubling the number of frequency modes to $2N$, producing a periodic extension of the CDF in which the spurious artifacts occupy the second half of the grid and can be discarded after the inverse transform. On this extended discrete grid, the Dirichlet kernel admits the closed-form expression
\begin{equation}
    \cal{D}(\omega_k) = \sum_{j=0}^{N-1} e^{i \omega_k j} = \frac{\sin \left( \tfrac{N \omega_k}{2} \right)}{\sin \left( \tfrac{\omega_k}{2} \right)} e^{i \omega_k \frac{N-1}{2}},
    \label{eq:dirichlet_kernel}
\end{equation}
evaluated on the extended frequency grid $\{\omega_k\}_{k=0}^{2N-1}$. We note that this kernel corresponds to a first-order, rectangular integration rule. In practice, superior accuracy can be achieved by replacing this expression with higher-order schemes.

\subsection{Spectral filtering}
\label{sec:spectral-filtering}

The Fourier spectral method enables band-limited distributions to be reconstructed exactly from finitely many frequency samples, while distributions with rapidly decaying spectra admit converging approximations as the number of modes $N$ increases. However, distributions with jump discontinuities exhibit a non-decaying Fourier spectrum. This phenomenon manifests as persistent Gibbs oscillations, arising in neighborhoods of the discontinuities and preventing uniform convergence of the reconstruction.

Spectral filtering provides a well-established remedy for suppressing these oscillations~\cite{vandeven1991, tadmor2007}. In this approach, smooth cutoff functions $\sigma(\omega)$ are applied to the characteristic function, attenuating high-frequency components. In physical space, these functions act as mollifiers, broadening the oscillations while preserving the overall structure of the distribution. A fundamental result due to Gottlieb and Shu states that, sufficiently away from discontinuities, filtered spectral approximations of the CDF $F^{N, \sigma}(x_k)$ converge at rates determined by the smoothness of the filter~\cite{gottlieb1997}. Polynomial filters of order $q$ yield algebraic pointwise convergence of order $\O{N^{1-q}}$, while exponential filters achieve exponential convergence $\O{e^{-N}}$. 

Near discontinuities, pointwise convergence cannot be achieved. Instead, Gibbs oscillations persist within narrow regions whose width decreases linearly with increasing $N$. This yields a rigorous bound for the global reconstruction error in the $\rm{L}^p$ norm,
\begin{equation}
    \|F - F^{N,\sigma}\|_{\rm{L}^p} = \O{N^{-1/p}},
    \label{eq:F_N_convergence_normp}
\end{equation}
independently of the specific filter employed. In particular, the mean pointwise error, defined by the $\rm{L}^1$ norm, decays as $\O{N^{-1}}$. Further details of these convergence properties are discussed in Appendix~\ref{appendix:filtering_convergence}.

\subsection{Spectral reconstruction algorithm}
\label{sec:spectral-algorithm}

Combining the previous ingredients, the spectral reconstruction of the density for weighted-sum models~\eqref{eq:X} can be expressed as
\begin{equation}
    f_X^{N, \sigma}(x_k) = \cal{F}^{-1} \left[ \prod_{d=1}^D \phi_{X_d}(w_d \omega) \cdot \sigma(\omega) \right](x_k),
    \label{eq:f_N_sigma}
\end{equation}
with the spectral filter $\sigma(\omega)$ being optional and required only for distributions with jump discontinuities. 

Similarly, the CDF can be obtained by incorporating the (zero-padded) Dirichlet kernel, as
\begin{equation}
    F_X^{N, \sigma}(x_k) = \cal{F}^{-1} \left[ \prod_{d=1}^D \phi_{X_d}(w_d \omega) \cdot \sigma(\omega) \cdot \cal{D}(\omega) \right](x_k).
    \label{eq:F_N_sigma}
\end{equation}
These expressions provide an efficient procedure for evaluating the target density and CDF of weighted-sum models~\eqref{eq:X}. We summarize the procedure for computing the CDF in Algorithm~\ref{alg:dense_spectral}.
\begin{figure}[t]
  \begin{algorithm}[H]
    \caption{Dense-vector spectral CDF approximation of a weighted-sum model~\eqref{eq:X}.}
    \label{alg:dense_spectral}
    \begin{algorithmic}[1]
        \Require Support width $L$, number of frequency modes $N$, local characteristic functions $\phi_{X_d}(\omega_k)$ and weights $w_d$ for $d = 1, \ldots, D$, spectral filter $\sigma(\omega_k)$.
        \Ensure Spectral approximation $F_X^{N, \sigma}$ of the CDF $F(x)$ on a grid of size $N$.
        \Statex
        \Statex \textit{Define zero-padded Fourier parameters}
        \State $\Delta x \gets \frac{L}{N}, \quad \Omega \gets \frac{\pi}{\Delta x}, \quad \Delta \omega_{\text{pad}} \gets \frac{\Omega}{N}$
        \State $\omega_k \gets - \Omega + k \Delta \omega_{\text{pad}}, \quad k = 0, 1, \ldots, 2N-1$
        \Statex
        \Statex \textit{Construct the (filtered) characteristic function}
        \State $\phi_X \gets \sigma(\omega_k)$
        \For{$d=1$ \textbf{to} $D$}
            \State $\phi_X \gets \phi_X \cdot \phi_{X_d}(w_d \omega_k)$
        \EndFor
        \Statex
        \Statex \textit{Convolve with the Dirichlet kernel~\eqref{eq:dirichlet_kernel}}
        \State $\cal{D} \gets \sum_{j=0}^{N-1}e^{i \omega_k j}$
        \State $\phi_X \gets \phi_X \cdot \cal{D}$
        \Statex
        \Statex \textit{Inverse Fourier transform and undo zero-padding}
        \State $F_X^{N, \sigma} \gets \cal{F}^{-1}(\phi_X)$
        \State \Return $F_X^{N, \sigma}[k], \quad k = 0, \ldots, N$ \Comment{Discard second half}
    \end{algorithmic}
  \end{algorithm}
\end{figure}

As demonstrated by the numerical experiments shown in Section~\ref{sec:results}, the spectral method systematically outperforms Monte Carlo sampling, achieving quadratic speedups with comparable accuracy. However, resolving the characteristic function to the prescribed accuracy may require a large number of frequency modes, especially if the distribution presents jump discontinuities or is heavy-tailed, leading to substantial computational and memory costs. This motivates the tensor network formulation introduced in the following section, which exploits low-rank structure in Fourier space to achieve polylogarithmic complexity in $N$.

\section{QTT-based Fourier spectral method}
\label{sec:qtt}

We now develop a \emph{quantum-inspired} algorithm that formulates the spectral method of Section~\ref{sec:spectral} in terms of quantized tensor trains (QTT), enabling computational advantages whenever the characteristic function exhibits low-rank structure. The term ``quantum-inspired'' refers to the use of the MPS/QTT formalism originally developed for quantum many-body physics~\cite{perez2006,verstraete2008,orus2014,cirac2021}, here adapted as a purely classical numerical method. The characteristic function is encoded in QTT format by exploiting its factorized structure into local components (see Eq.~\eqref{eq:cf_factorization}), which are combined through a sequence of QTT Hadamard products.

The efficiency of the QTT framework stems from the low-rank structure of the global characteristic function in frequency space. In continuous models, this follows directly from the smooth, rapidly decaying local spectra, consistent with the known QTT compressibility of smooth functions~\cite{garcia2021, jobst2024}. In discrete models, compressibility arises from a concentration of spectral energy induced by multiplicative suppression of high-frequency quasiperiodic modes. Moreover, as $D$ increases, the local characteristic functions are sampled over a progressively narrower frequency region, becoming increasingly smooth and reducing the Hadamard construction cost. Together, these two mechanisms drive the spectrum toward a smooth, low-rank envelope. Importantly, this spectral concentration does not require the underlying distribution to be near Gaussian, enabling the method to capture the non-Gaussian features of the target distribution.

\subsection{Quantized tensor trains}
\label{sec:qtt-qtt}

\begin{figure}[t]
    \centering
    \includegraphics[width=0.8\linewidth]{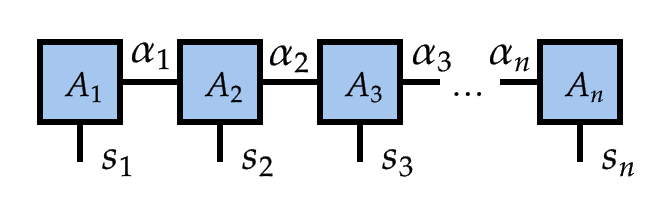}
    \caption{QTT diagram with $n$ rank-3 cores, carrying free indices $s_k$ of size $\ell_k$ and mutually contracted along virtual bonds $\alpha_k$ of dimension $\chi_k$.}
    \label{fig:qtt}
\end{figure}

Tensor trains (TT)~\cite{oseledets2011}, also known as matrix-product states (MPS)~\cite{perez2006, verstraete2008, orus2014, cirac2021}, decompose high-order tensors into a sequence of mutually contracted, lower-order tensors, known as TT cores. Let $A_{s_1 s_2 \ldots s_n}$ denote an order-$n$ tensor---i.e., a numerical array with indices $s_k \in \{0, \ldots, \ell_k - 1\}$. The storage and computational cost associated with such a tensor grows exponentially with $n$, quickly rendering direct numerical treatment infeasible. 

In the TT representation, $A$ is decomposed as
\begin{equation}
    A_{s_1 s_2 \ldots s_n} = \sum_{\alpha_1=1}^{\chi_1}\sum_{\alpha_2=1}^{\chi_2} \!\! \cdots \!\! \sum_{\alpha_{n-1}=1}^{\chi_{n-1}} \!\! A^{s_1}_{\alpha_1} A^{s_2}_{\alpha_1 \alpha_2} \cdots A^{s_n}_{\alpha_{n-1}},
    \label{eq:A_tt}
\end{equation}
with auxiliary indices $\alpha_k$ having dimension $\chi_k$, known as the bond dimensions or TT-ranks. As shown in Fig.~\ref{fig:qtt}, this decomposition is typically depicted graphically as a one-dimensional chain of boxes, with edges representing the auxiliary indices $\alpha_k$ and open legs corresponding to the free indices $s_k$.

By construction, the total number of parameters in the TT~\eqref{eq:A_tt} scales polynomially with the tensor order as $\O{n \ell \chi_{\max}^2}$, where $\ell := \max_k \ell_k$ and $\chi_{\max} := \max_k \chi_k$. When the bond dimensions remain moderate as $n$ grows, this representation provides an exponential reduction in storage and computational complexity compared to the explicit tensor $A_{s_1 s_2 \ldots s_n}$.

A quantized or ``quantics'' tensor train (QTT) refers to the TT decomposition of a discretized function after reshaping the grid index into an $\ell$-ary representation~\cite{khoromskij2011}. Let $x_s$ denote a continuous variable discretized on an equispaced grid over the interval $[a, b]$. The grid index $s$ can be mapped to a multi-index $(s_1, \ldots, s_n)$ through the $\ell$-ary encoding
\begin{equation}
    x_{s_1 s_2 \ldots s_n} = a + (b-a) \sum_{i=1}^n s_i \ell^{-i},
    \label{eq:x_ellary}
\end{equation}
where the digits satisfy $s_k \in \{0, \ldots, \ell_k - 1\}$. The resulting QTT---i.e., the TT representation of the tensor~\eqref{eq:x_ellary}---represents an equispaced grid with $N = \ell^n$ points using only $n = \log_\ell N$ indices. Consequently, increasing the number of QTT cores $n$ corresponds to an exponential refinement of the underlying grid. This construction extends naturally to multivariate functions $f_{s_1 s_2 \ldots s_n}$ by applying the encoding in Eq.~\eqref{eq:x_ellary} independently along each spatial dimension.

For convenience, we adopt the bra-ket notation commonly used in the quantum physics literature~\cite{orus2014, cirac2021}. In this notation, a QTT tensor $A_{s_1 s_2 \ldots s_n}$ is represented as a vector or ``ket'', $\ket{A}$, in a tensor-product Hilbert space,
\begin{equation}
    \ket{A} = \sum_{s_1=1}^{\ell_1} \cdots \sum_{s_n=1}^{\ell_n} A_{s_1 s_2 \ldots s_n} \ket{s_1} \otimes \cdots \otimes \ket{s_n},
    \label{eq:A_state}
\end{equation}
where $\{\ket{s_k}\}_{k=1}^{\ell_k}$ denotes the canonical basis associated with the $k$-th mode. Similarly, a ``bra'' vector $\bra{B}$ represents another element of the Hilbert space whose inner product $\braket{B | A}$ corresponds to the contraction of the tensors over all shared indices. 

Currently, MPS/QTT stands as a growing platform for \emph{quantum-inspired} algorithms~\cite{garcia2021}. These methods have been successfully applied across a wide range of scientific computing problems, including Fourier analysis~\cite{dolgov2012,garcia2021,garcia2022,chen2023}, high-dimensional differential equations~\cite{dolgov2012b, gourianov2022, garcia2023, gidi2024}, and multivariate numerical integration~\cite{dolgov2020, nunez2022, alexandrov2023}. When applicable, such approaches can yield substantial reductions in computational cost and memory requirements compared with conventional techniques.

\subsection{Finite-precision algebra}
\label{sec:qtt-algebra}

Tensor trains admit an algebra of finite-precision operations~\cite{garcia2021}. Since these operations generally increase their bond dimensions, they are generally followed by a truncation step to maintain computational efficiency, introducing a controlled approximation error. Several efficient truncation procedures exist, usually based on the TT-SVD algorithm~\cite{oseledets2010,stoudenmire2010,orus2014,paeckel2019}. We employ a two-core variational scheme that controls the truncation error in $\rm{L}^2$ norm through a tolerance parameter $\epsilon$~\cite{garcia2023,rodriguez2024}. The cost of the truncation step scales as $\O{\ell n \chi_{\max}^3}$.

In this work, we require two elementary QTT operations. The first is the Hadamard product, defined as the elementwise multiplication of two QTT tensors,
\begin{equation}
    \ket{u} \odot \ket{v} = \sum_{s_1,s_2,\ldots,s_n} u_{s_1s_2\ldots s_n}\, v_{s_1s_2\ldots s_n}\, \ket{s_1,s_2,\ldots, s_n}.
    \label{eq:qtt_hadamard}
\end{equation}
The resulting QTT presents a multiplicative bond dimension $\chi_u \cdot \chi_v$, requiring a subsequent truncation step. Although alternative approaches based on iterated core-swapping~\cite{michailidis2025} or sketching methods~\cite{sun2024,meng2026} have been proposed, in practice we found the direct approach to be simpler and reasonably competitive.

The second operation is the tensorized or ``superfast'' Fourier transform~\cite{dolgov2012}. This procedure maps a tensor train to its Fourier transform through a linear QTT operator $\widehat{\cal{F}}$, enabling large speedups when the input QTT has low rank. This operator is directly analogous to the quantum Fourier transform (QFT)~\cite{coppersmith2002, nielsen2010}, which can be implemented as a quantum circuit of unitary gates. Moreover, the QFT is known to exhibit low entanglement when the bit-reversal permutations are separated from the circuit~\cite{chen2023,chen2024}, enabling a low-rank operator representation.

For a more comprehensive description of the TT algebra and its numerical implementation, we refer the reader to Ref.~\cite{seemps2026}, which documents the numerical library \textsc{SeeMPS} used in this work.

\subsection{Fourier spectral approximation using QTT}
\label{sec:qtt-spectral}

We now formulate the Fourier spectral method of Sect.~\ref{sec:spectral} using tensor networks, enabling the evaluation of the density~\eqref{eq:f_N_sigma} and CDF~\eqref{eq:F_N_sigma} of weighted-sum models directly in QTT form. The method proceeds in two stages: first, it encodes the filtered characteristic function (Sect.~\ref{sec:qtt-spectral-cf}), followed by convolution with the Dirichlet kernel and Fourier inversion (Sect.~\ref{sec:qtt-spectral-convolution}).

\subsubsection{Construction of the filtered characteristic function}
\label{sec:qtt-spectral-cf}

The global characteristic function $\phi_X$~\eqref{eq:cf_factorization} of the model is encoded in QTT format by combining the local contributions $\ket{\phi_{X_d}}$ through a sequence of Hadamard products,
\begin{equation}
    \ket{\phi_X} = \prod_{d=1}^D \ket{\phi_{X_d}},
    \label{eq:qtt_cf_sum}
\end{equation}
which are optionally combined with a spectral filter $\ket{\sigma}$. 

As shown in Eq.~\eqref{eq:phi_X_multinomial}, discrete random variables with cardinality $K_d$ admit exact QTT characteristic functions $\ket{\phi_{X_d}}$ with bond dimension $K_d$. More generally, these local components can be encoded using tensor cross-interpolation~(TCI), a widely used method for constructing TT representations of black-box functions via adaptive sampling~\cite{oseledets2010,dolgov2020,ritter2024}. The spectral filter $\ket{\sigma}$, being a smooth unidimensional function of frequency, can be constructed in a stable and convergent manner via an orthogonal polynomial expansion in the Chebyshev basis, yielding uniformly accurate approximations with relatively low polynomial order~\cite{rodriguez2024}.

To prevent unbounded growth of the bond dimensions, truncation with tolerance $\epsilon$ is applied after each Hadamard product. To further enhance compressibility, the local characteristic functions are applied sequentially to the initial spectral filter $\ket{\sigma}$, producing the filtered characteristic function $\ket{\phi_X} \odot \ket{\sigma}$ directly without requiring a final filtering step. As shown in Sect.~\ref{sec:results-discrete}, this strategy empirically yields smoother bond dimension profiles and improved computational efficiency, since high-frequency components that would increase the bond dimension are dampened from the start.

Although we compute Eq.~\eqref{eq:qtt_cf_sum} sequentially, the product is commutative and its evaluation is embarrassingly parallel. In particular, as explained in Sect.~\ref{sec:qtt-complexity}, a binary-tree reduction scheme can be employed across $\O{D}$ independent processors to reduce the time complexity.

\subsubsection{Dirichlet convolution and Fourier inversion}
\label{sec:qtt-spectral-convolution}

The remaining algorithmic steps, which are purely structural and model-independent, consist of convolving the filtered characteristic function $\ket{\phi_X} \odot \ket{\sigma}$ with the Dirichlet kernel, followed by Fourier inversion to obtain the CDF in Eq.~\eqref{eq:F_N_sigma}.

To first order, the Dirichlet kernel~\eqref{eq:dirichlet_kernel} arises as the discrete Fourier transform of a unit step function. Accordingly, it admits the QTT representation
\begin{equation}
    \ket{\cal{D}} := \sum_{j=0}^{N-1} \ket{e^{i \omega j}}.
    \label{eq:qtt_dirichlet}
\end{equation}
which can be obtained by applying the Fourier QTT operator $\widehat{\cal{F}}$ to an exact rank-two encoding of the unit step. Higher-order kernels arise from the Fourier transform of periodic vectors, which can be implemented analogously to Eq.~\eqref{eq:qtt_dirichlet}.

Subsequently, the physical CDF is obtained through the tensorized inverse Fourier transform, yielding an expression
\begin{equation}
    \ket{F_X^{N,\sigma}} = \widehat{\cal{F}}^{-1} \left( \ket{\phi_X} \odot \ket{\sigma} \odot \ket{\cal{D}} \right)
    \label{eq:qtt_F_spectral}
\end{equation}
that is in complete analogy to the dense formulation in Eq.~\eqref{eq:F_N_sigma}.

These steps are summarized in Algorithm~\ref{alg:qtt_spectral}, relying on the following routines:
\begin{enumerate}
    \item \textsc{Truncate}: truncates the bond dimensions of a QTT via a variational TT-SVD routine with tolerance $\epsilon$.
    \item \textsc{SpectralFilter}: encodes the QTT of the chosen spectral filter using $n$ QTT cores and cutoff $\Omega$.
    \item \textsc{LocalCF}: encodes the QTT of the local characteristic function (e.g., via exact construction or TCI) using $n$ QTT cores and frequency cutoff $\Omega$.

\end{enumerate}

\begin{figure}[t]
  \begin{algorithm}[H]
    \caption{QTT-based spectral CDF approximation of a weighted-sum model~\eqref{eq:X}.}
    \label{alg:qtt_spectral}
    \begin{algorithmic}[1]
      \Require Support width $L$, number of cores $n$, local characteristic functions $\phi_{X_d}(\omega_k)$ and weights $w_d$ for $d = 1, \ldots, D$, spectral filter $\sigma(\omega_k)$, truncation tolerance $\epsilon$.
      \Ensure QTT approximation $\ket{F_X^{N,\sigma}}$ of the CDF $F(x)$ on a grid of size $N=2^n$.
      \Statex
      \Statex \textit{Define zero-padded Fourier parameters}
      \State $N \gets 2^n$, \quad $n_{\text{pad}} \gets n+1$, \quad $\Delta x \gets \frac{L}{N}$, \quad $\Omega \gets \frac{\pi}{\Delta x}$
      \Statex
      \Statex \textit{Construct the global characteristic function (Sect.~\ref{sec:qtt-spectral-cf})}
      \State $\ket{\sigma} \gets \textsc{SpectralFilter}(\sigma(\omega), \ n_{\text{pad}}, \ \Omega, \ \epsilon)$
      \State $\ket{\phi_X} \gets \ket{\sigma}$
      \For{$d = 1$ \textbf{to} $D$}
        \State $\ket{\phi_{X_d}} \gets \textsc{LocalCF}(\phi_{X_d}(w_d \omega), \ n_{\text{pad}}, \ \Omega)$
        \State $\ket{\phi_X} \gets \textsc{Truncate}\left(\ket{\phi_X} \odot \ket{\phi_{X_d}}; \ \epsilon\right)$
      \EndFor
      \Statex
      \Statex \textit{Convolve with the Dirichlet kernel (Sect.~\ref{sec:qtt-spectral-convolution})}
      \State $\ket{\cal{D}} \gets \sum_{j=0}^{N-1} \ket{e^{i \omega j}}$
      \State $\ket{\phi_X} \gets \textsc{Truncate}\left(\ket{\phi_X} \odot \ket{\cal{D}}; \ \epsilon\right)$
      \Statex
      \Statex \textit{Inverse Fourier transform and undo zero-padding}
      \State $\ket{F_X^{N, \sigma}} \gets \widehat{\cal{F}}^{-1} \ket{\phi_X}$
      \State \Return $(\bra{0} \otimes I^{\otimes n}) \ket{F_X^{N, \sigma}}$ \Comment{Project first core to $\ket{0}$}
    \end{algorithmic}
  \end{algorithm}
\end{figure}

\subsection{Computational complexity}
\label{sec:qtt-complexity}

The cost of the Fourier spectral approximation depends on whether the method is implemented using dense vectors or QTT tensors. In the dense setting, the cost scales algebraically with the number of frequency modes $N$, which becomes expensive when the spectrum decays slowly (e.g., in the presence of Gibbs oscillations). By contrast, when a compressible QTT representation is available, the cost scales linearly in the number of QTT cores $n = \log_2 N$, i.e., polylogarithmically in $N$.

\subsubsection{Complexity for dense-vector Fourier approximations}

The dense-vector spectral approximation described in Eq.~\eqref{eq:F_N_sigma} consists of four steps:
\begin{enumerate}
    \item \textbf{Characteristic function.}
    Evaluating $\phi_{X_d}(\omega)$ for $D$ independent components on $2N$ frequency samples costs $\O{DN}$.
    \item \textbf{Spectral filtering (optional).}
    Applying the filter $\sigma(\omega)$ costs $\O{N}$.
    \item \textbf{Convolution.}
    Multiplication by the Dirichlet kernel $\cal{D}(\omega)$ costs $\O{N}$.
    \item \textbf{Fourier inversion.}
    The inverse FFT dominates the computation with a cost $\O{N \log N}$.
\end{enumerate}
Combining these contributions yields the total cost
\begin{align}
    \rm{Time}(D, N) &= \O{DN + N\log N}, \\
    \rm{Memory}(N) &= \O{N}.
    \label{eq:computational_complexity_general}
\end{align}
The effective complexity thus depends on the scaling of the number of modes with the target tolerance, $N(\varepsilon)$.

\paragraph{Continuous densities.}
For continuous densities, the truncated spectral reconstruction reads
\begin{equation}
    f^{N}(x) = \frac{1}{2\pi} \int_{|\omega| \leq \Omega} \phi(\omega) \ e^{-i \omega x} \rm{d}\omega.
\end{equation}
Assuming a Gaussian decay of the characteristic function, $|\phi(\omega)| \leq e^{-\omega^2}$, the truncation error satisfies 
\begin{equation}
    \|f - f^{N}\|_{\rm{L}^\infty} \leq \frac{e^{-\Omega^2}}{\Omega}.
    \label{eq:density_error_continuous}
\end{equation}
This bound also controls the CDF error, implying an efficient scaling
\begin{equation}
    N(\varepsilon_\rm{cont}) = \O{\sqrt{\log \varepsilon_\rm{cont}^{-1}}}.
    \label{eq:convergence_N_cont}
\end{equation}

\paragraph{Discontinuous distributions.}
For distributions with jump discontinuities, Gibbs oscillations limit global accuracy. From Eq.~\eqref{eq:F_N_convergence_normp}, the global $\rm{L}^1$ error decays algebraically,
\begin{equation}
    N(\varepsilon_\rm{dis}) = \O{\varepsilon_\rm{dis}^{-1}},
    \label{eq:convergence_N_disc}
\end{equation}

Away from discontinuities, the convergence rate depends on the spectral filter. Polynomial filters of order $q$ yield
\begin{equation}
    N(\varepsilon_\rm{poly}) = \O{\varepsilon_\rm{poly}^{-1/q}}, 
    \label{eq:convergence_N_poly}
\end{equation}
while exponential filters achieve
\begin{equation}
    N(\varepsilon_\rm{exp}) = \O{\log \varepsilon_\rm{exp}^{-1}}. \label{eq:convergence_N_exp}
\end{equation}

These convergence rates illustrate the advantage of spectral methods over Monte Carlo sampling. Even in the worst Gibbs-dominated regime, spectral methods outperform Monte Carlo sampling, whose error scales as $\varepsilon(S) \sim S^{-1/2}$ and therefore requires $S(\varepsilon) = \O{\varepsilon^{-2}}$ samples (see Appendix~\ref{appendix:conventional-montecarlo}).

\subsubsection{Complexity for QTT-based Fourier approximations}

In the QTT formulation, the dominant cost arises from constructing the filtered characteristic function $\ket{\phi_X} \odot \ket{\sigma}$ in Eq.~\eqref{eq:qtt_cf_sum}.

Let $\bar{\chi} \ll \chi_{\max}$ denote the saturation bond dimension reached during the sequence of Hadamard products. The resulting complexity is
\begin{align}
\rm{Time}(D) &= \O{n D \bar{\chi}^3}\\
\rm{Memory}(D) &= \O{n \bar{\chi}^2},
\end{align}
reflecting the cubic cost of TT-SVD truncation and the linear dependence on the number of QTT cores. All remaining steps---Dirichlet convolution and tensorized Fourier transforms---operate within bounded-rank tensors and therefore require only subleading resources.

In the slowest convergence regime, Eq.~\eqref{eq:convergence_N_disc} implies $n = \O{\log \varepsilon_\rm{dis}^{-1}}$. Consequently, the complexity becomes
\begin{align}
\rm{Time}(D,\varepsilon_\rm{dis}) &= \O{\log \varepsilon_\rm{dis}^{-1} \ D \bar{\chi}^3},\\
\rm{Memory}(D,\varepsilon_\rm{dis}) &= \O{\log \varepsilon_\rm{dis}^{-1} \ \bar{\chi}^2}.
\end{align}
Analogously, in the fastest regime in Eq.~\eqref{eq:convergence_N_cont}, a $\O{\log \log \varepsilon_\rm{cont}^{-1}}$ complexity is achieved. Additionally, by employing a binary tree reduction scheme over $\O{D}$ independent processors, the time complexity with respect to the number of components can be reduced to $\O{\log D}$, yielding an effective parallel runtime of $\rm{Time}_\rm{parallel}(D) = \O{n \log D \ \bar{\chi}^3}$.

As is standard in tensor-network methods, the effective bond dimension $\bar{\chi}$ is not known a priori but is determined empirically by the truncation procedure itself. This verification is inexpensive: the algorithm reveals whether compression is effective as a byproduct of its execution, and incompressible regimes are detected immediately through bond-dimension growth. When the characteristic function is compressible and $\bar{\chi}$ remains bounded, the QTT spectral method achieves a polylogarithmic dependence in the target accuracy and can yield exponential savings compared with dense-vector spectral methods. In incompressible regimes, $\bar{\chi}$ grows to its maximal value and the complexity scaling approaches that of dense arrays.

\subsection{Post-processing operations}
\label{sec:qtt-postprocessing}

Several post-processing operations can be applied directly on compressed QTT representations without materializing the underlying dense vector. Two representative examples are quantile estimation (Sect.~\ref{sec:qtt-postprocessing-var}) and masked integration (Sect.~\ref{sec:qtt-postprocessing-es}), which allow the efficient computation of the Value at Risk~\cite{jorion2007,embrechts2013b} and Expected Shortfall~\cite{artzner1999,acerbi2002} risk metrics.

\subsubsection{Quantile estimation and Value at Risk}
\label{sec:qtt-postprocessing-var}

\begin{figure}[t]
  \begin{algorithm}[H]
    \caption{Binary search on a monotone function encoded as MPS.}
    \label{alg:qtt_binary_search}
    \begin{algorithmic}[1]
      \Require Binary QTT cores $\{A_k^{s_k}\}_{k=1}^n$ encoding a monotone function $F(x)$; target value $\alpha$.
      \Ensure Bit string $(s_1,\ldots,s_n)$ such that $F(x_{s_1\ldots s_n}) \approx \alpha$

      \Statex
      \Statex \textit{Precompute right environment vectors}
      \State $R_{n+1} \gets \mathbf{1}$
      \For{$k = n, n-1, \ldots, 1$}
        \State $R_k \gets A_k^{1} R_{k+1}$
      \EndFor
      \Statex
      \Statex \textit{Left-to-right binary search}
      \State $L \gets \mathbf{1}$
      \For{$k = 1, 2, \ldots, n$}
        \State $y_k \gets L \ A_k^{0} \ R_{k+1}$
        \State $s_k \gets 0$ \textbf{if} $y_k \ge \alpha$ \textbf{else} $1$
        \State $L \gets L A_k^{s_k}$
      \EndFor
      \State \Return $(s_1, s_2, \ldots, s_n)$

    \end{algorithmic}
  \end{algorithm}
\end{figure}

Given a CDF $F(x)$, quantile estimation at level $\alpha$ consists of evaluating $F^{-1}(\alpha)$, that is, the smallest grid point $x_k$ such that $F(x_k) \geq \alpha$. 

When the monotone function is represented in binary QTT form, this operation can be performed directly on the QTT cores through a left-to-right binary search. At each step, a partial contraction determines whether the next digit of the expansion should be set to zero or one, yielding the multi-index corresponding to the desired quantile. The procedure scales linearly with the number of tensor cores and is summarized in Algorithm~\ref{alg:qtt_binary_search}.

In risk analysis, the VaR~\cite{jorion2007, embrechts2013b} of a loss variable $X_{\rm{L}}$ is defined as the corresponding quantile of its CDF,
\begin{equation}
    \rm{VaR}_\alpha(X_\rm{L}) = F_{X_\rm{L}}^{-1}(\alpha).
    \label{eq:VaR}
\end{equation}
This quantity can be estimated directly from the compressed QTT representation of the CDF. The computational cost scales as $\O{n \chi_{\max}^2}$, slightly larger than the $\O{n}$ cost of binary search on dense arrays, but without the exponential cost of reconstructing the dense vector.

\subsubsection{Restricted integration and Expected Shortfall}
\label{sec:qtt-postprocessing-es}

Given a QTT encoding of a function $F(x)$ on a uniform grid $\{x_k\}_{k=0}^{N-1} \in [0, L]$, its integral reduces to an inner product with a QTT encoding of some quadrature weights $\{w_k\}_{k=0}^{N-1}$,
\begin{equation}
    \int_0^L F(x) \mathrm{d}x \ \approx \ \braket{w | F} = \sum_{k=0}^{N-1} w_k F(x_k),
\end{equation}
which can be evaluated with a cost $\O{\chi_F^2 +\chi_w^2}$. On equispaced grids, Newton--Cotes quadrature weights admit exact low-rank QTT representations with bond dimension proportional to the quadrature order~\cite{alexandrov2023,seemps2026}. Restricted integrals can be computed by masking the integrand with a step function, with minor adjustments to the boundary weights to avoid an $\O{N^{-1}}$ bias.

Expected Shortfall~\cite{artzner1999, acerbi2002} is defined as the tail expectation
\begin{equation}
    \rm{ES}_\alpha(X_\rm{L}) = \frac{1}{1-\alpha} \int_{\rm{VaR}_\alpha}^{\infty} [1-F_{X_\rm{L}}(x)] \rm{d}x.
\end{equation}
Using the QTT representation of the CDF, this quantity can be approximated as 
\begin{equation}
    \rm{ES}_\alpha(X_\rm{L}) \approx \frac{\Delta x}{1-\alpha} \bra{\Theta_{\rm{VaR}_\alpha}} \left(\ket{1} - \ket{F_{X_\rm{L}}^{N,\sigma}}\right),
\end{equation}
where $\ket{\Theta_{\rm{VaR}_\alpha}}$ denotes the QTT encoding of the unit step function supported beyond $\rm{VaR}_\alpha$~\eqref{eq:VaR}.

\section{Numerical results}
\label{sec:results}

We illustrate the performance of the proposed Fourier spectral algorithm on two representative problem classes. In both cases, the method approximates the exact, non-Gaussian target distributions, while the computational savings over the dense-vector method stem from the low-rank CF structure identified in Sect.~\ref{sec:qtt}.

As a representative discrete model, we consider weighted sums of Bernoulli random variables, a model known as the weighted Poisson--binomial (WPB) distribution~\cite{chen1997,fernandez2010,hong2013}.  Despite its simple definition, the WPB model exhibits an exponentially growing support when the weights are incommensurate, serving as a representative instance of more general discrete models with heterogeneous supports~\cite{daskalakis2015}. As mentioned in the introduction, such models are widespread in applications including credit risk aggregation~\cite{vasicek1987, glasserman2005, bluhm2016} and reliability analysis~\cite{lisnianski2010, li2012}. In our context, it serves as a natural testbed for studying QTT compressibility in the large-$D$ regime.

As a representative continuous model, we consider sums of lognormal random variables. This model is central to applications with multiplicative sources of uncertainty, such as equity portfolio risk analysis~\cite{osborne1959, dufresne2004} and interference modeling in wireless communication systems~\cite{fenton1960, schwartz1982, mehta2007}. From a numerical perspective, sums of lognormals constitute a classical challenging problem due to their heavy tails and the exponentially large spatial domain required for accurate resolution. Furthermore, its characteristic function is not available in closed form and is given only as an integral equation. We address this by applying Gauss--Hermite quadrature to Gubner's integral representation~\cite{gubner2006}, which reduces the characteristic function to a weighted sum of complex exponentials, yielding an exact algebraic QTT construction. More generally, when such explicit representations are not available, TCI provides a general fallback, requiring only black-box access to the characteristic function.

\subsection{Discrete models}
\label{sec:results-discrete}

Weighted sums of discrete components $X_d$ produce an aggregate random variable $X$ defined on a finite set $\{x_m\}_{m=0}^{M-1}$ of cardinality $M$. Its probability law is described by a probability mass function (PMF) $f_{X, m} := \Pr[X = x_m]$, satisfying $\sum_m f_{X,m} = 1$. Formally, it corresponds to the singular density
\begin{equation}
    f_X(x) = \sum_{m=0}^{M-1} f_{X,m} \delta(x - x_m),
    \label{eq:f_X_singular}
\end{equation}
which is ill-posed for numerical treatment. Instead, it is preferable to work with the CDF~\eqref{eq:F_X}, which takes the piecewise-constant form
\begin{equation}
    F_X(x) = \sum_{m=0}^{M-1} f_{X,m} \Theta(x - x_m),
    \label{eq:F_X_discrete}
\end{equation}
and provides a numerically stable target for computation.

In full generality, discrete models arise from linear combinations of independent \emph{categorical} components
\begin{equation}
    X_d \sim \text{Categorical}(\cal{X}_d, \cal{P}_d),
    \label{eq:X_d_categorical}
\end{equation}
where each variable takes $K_d$ values on a finite support $\cal{X}_d = \{x_{1d}, \ldots, x_{K_d d}\}$ with probabilities $\cal{P}_d = \{p_{1 d}, \ldots, p_{K_d d}\}$. 

The resulting probability distribution is obtained by summing the probabilities of all categorical configurations that produce the same aggregate outcome, which, except for a few highly structured cases, does not admit a closed-form expression. The support is generated by enumerating all configurations and may contain up to $\prod_d K_d$ distinct values, which grows exponentially with the number of components $D$ unless degeneracies occur in the weights.

The characteristic function of the categorical variable~\eqref{eq:X_d_categorical} admits the closed-form expression
\begin{equation}
    \phi_{X,d}(\omega) = \sum_{j=1}^{K_d} p_{j d} e^{i \omega w_d x_{j d}}.
    \label{eq:phi_X_multinomial}
\end{equation}
Being a weighted sum of complex exponentials, this expression admits an exact QTT representation with bond dimension $\chi_d = K_d$.

Discrete models encompass several classical discrete laws. In particular, when $K_d = 2$, the categorical variables reduce to Bernoulli indicators $I_d \in \{0, 1\}$, yielding the binomial and weighted Poisson--binomial (WPB) models.

\paragraph{Binomial model.}
The binomial model arises as the simplest discrete aggregate model~\cite{feller1991}, obtained from the sum of $D$ i.i.d. Bernoulli variables $I \sim \mathrm{Bernoulli}(p)$ with unit weights,
\begin{equation}
    X_{\rm{B}} = \sum_{d=1}^D \ I.
    \label{eq:X_B}
\end{equation}
Its distribution is highly degenerate, with equispaced support $\{0, 1, \ldots, D\}$ of cardinality $M = D+1$, and admits the closed-form PMF,
\begin{equation}
    f_{\rm{B},m} = \binom{D}{m} p^m (1-p)^{D-m}.
    \label{eq:f_B}
\end{equation}
This analytical tractability makes the binomial distribution a natural reference case for validating numerical methods.

\paragraph{Weighted Poisson-binomial model.}
Allowing heterogeneous probabilities and weights leads to the weighted Poisson--binomial (WPB) model~\cite{chen1997, fernandez2010},
\begin{equation}
    X_{\rm{WPB}} = \sum_{d=1}^D w_d I_d, \qquad I_d \sim \text{Bernoulli}(p_d).
    \label{eq:X_WPB}
\end{equation}
Unlike the binomial case, this model generally does not admit a closed-form PMF. The structure of its support depends on the commensurability of the weights. When all weights are equal, the support has minimal size $M=D+1$, while for generic incommensurate weights, it may reach the maximal size $M = 2^D$, becoming dense and irregular. Hence, the WPB model captures the full range of behaviors encountered in more general sums of categorical variables, from highly degenerate to maximally irregular regimes.

\subsubsection{Experimental setup}

\begin{figure}[t]
    \centering
    \includegraphics{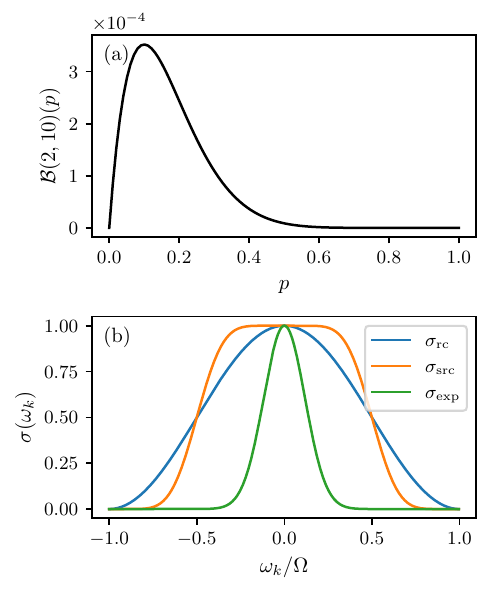}
    \caption{
        Experimental setting used for the discrete experiments.
        (a)~ Beta $\cal{B}(2, 10)$ distribution used to sample the weights.
        (b)~Spectral filters considered, including the raised cosine filter $\sigma_{\mathrm{rc}}$ (blue, $q=2$), sharpened raised cosine filter $\sigma_{\mathrm{src}}$ (orange, $q=8$), and exponential (Gaussian) filter $\sigma_{\mathrm{exp}}$ (green).
    }
    \label{fig:filters}
\end{figure}

We report numerical results for discrete models, focusing on instances of the weighted Poisson--binomial model~\eqref{eq:X_WPB} normalized to the unit interval, $X_{\rm{WPB}} \in [0, 1]$. In all experiments, the weights are drawn from a uniform distribution $w_d \sim \cal{U}(0, 1)$ and normalized such that $\sum_d w_d = 1$. The probabilities $p_d$ are sampled independently from a Beta distribution $\cal{B}(2, 10)$ shown in Fig.~\ref{fig:filters}a. The numerical results are averaged over ensembles of ten independent realizations, reporting both the mean behavior together with error bars for one standard deviation.

Since discrete distributions exhibit jump discontinuities in their CDF, spectral filtering is required to control Gibbs oscillations. We consider three representative filters, illustrated in Fig.~\ref{fig:filters}b: the raised cosine filter $\sigma_{\rm{rc}}$ of order $q=2$, its sharpened variant $\sigma_{\rm{src}}$ of order $q=8$, and a Gaussian filter $\sigma_{\rm{exp}}$, defined as
\begin{align}
    \sigma_{\rm{rc}}(\eta) &= \frac{1}{2}[1 + \cos(\pi \eta)], \label{eq:filter-rc}\\
    \sigma_{\rm{src}}(\eta) &= \left[\sigma_{\rm{rc}}^4 \ (35 - 84 \sigma_{\rm{rc}} + 70 \sigma_{\rm{rc}}^2 - 20 \sigma_{\rm{rc}}^3)\right](\eta), \label{eq:filter-src}\\
    \sigma_{\rm{exp}}(\eta) &= e^{-\alpha \eta^2}, \label{eq:filter-exp}
\end{align}
defined in terms of a normalized frequency $\eta := \omega / \Omega$, where $\alpha$ is determined so that $\sigma_{\rm{exp}}$ is below machine precision at $\eta = \pm 1$. As discussed in Sect.~\ref{sec:spectral-filtering}, the choice of filter determines the spectral convergence rate away from the jump discontinuities.

\subsubsection{Spectral convergence}

\begin{figure}[t]
    \centering
    \includegraphics{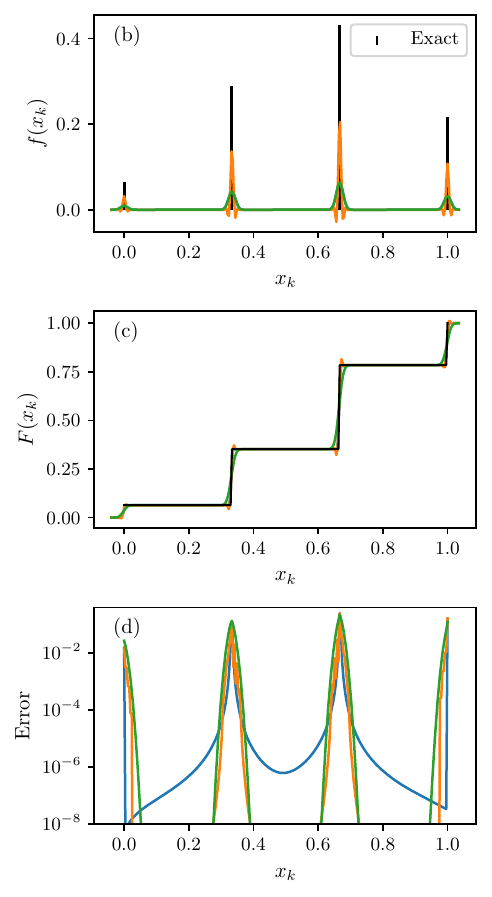}
    \caption{Spectral reconstruction of the binomial distribution with $D=3$ Bernoulli variables and 256 frequency modes, using the filters shown in Fig.~\ref{fig:filters}b.
    (a)~Exact PMF and spectral reconstructions. 
    (b)~Reconstructed CDF.
    (c)~Pointwise error of the CDF, showing Gibbs oscillations surrounding the jump discontinuities.}
    \label{fig:binomial_D3_distributions}
\end{figure}
\begin{figure}[t]
    \centering
    \includegraphics{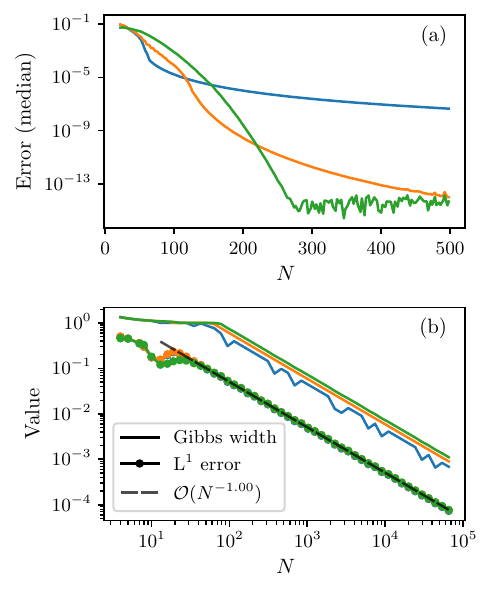}
    \caption{Error analysis for a binomial model with $D=3$ components. (a)~Decay of the median error (50\% quantile of pointwise errors), reflecting the filter-dependent spectral convergence rates in smooth regions. (b)~Decay of the global $\rm{L}^1$ error and the widths of the Gibbs bands.}
    \label{fig:binomial_D3_convergence}
\end{figure}

To build intuition, we first consider a binomial model~\eqref{eq:f_B} with $D=3$ i.i.d. Bernoulli variables. The PMF consists of four equispaced point masses, and the CDF exhibits four jump discontinuities. As expected, Gibbs oscillations dominate near discontinuities, while convergence in smooth regions improves with increasing filter order, as confirmed by the spectral reconstructions of the PMF and CDF for $N=256$ frequency modes and the pointwise CDF error (Figs.~\ref{fig:binomial_D3_distributions}a--c). 

The convergence with respect to $N$ follows the theoretical expectations. To isolate smooth-region behavior, we analyze quantiles of the pointwise error distribution. The median (50\% quantile) error converges algebraically for the two polynomial filters and exponentially for the Gaussian filter (Fig.~\ref{fig:binomial_D3_convergence}a). By contrast, the global error is governed by non-decaying Gibbs oscillations, reflected in the $\O{N^{-1}}$ decay of the global $\rm{L}^1$ error and the shrinking width of the Gibbs bands (Fig.~\ref{fig:binomial_D3_convergence}b).

\begin{figure}[t]
    \centering
    \includegraphics{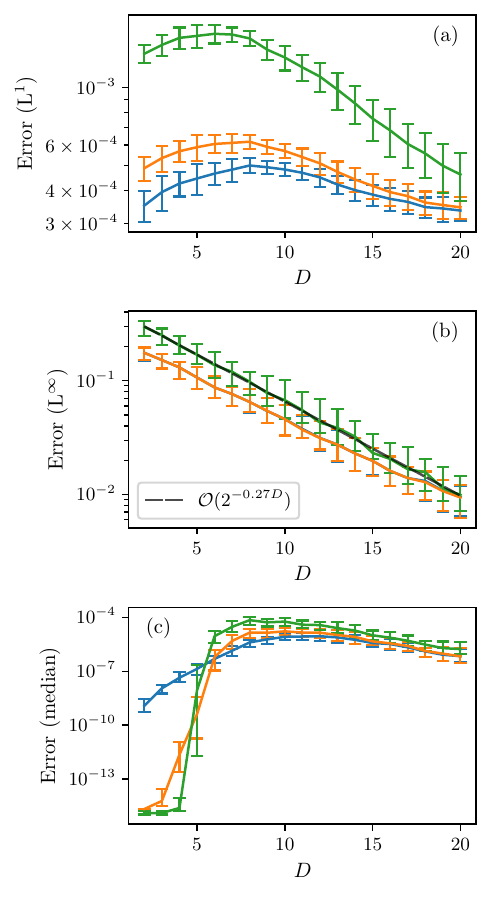}
    \caption{Error analysis for WPB models of increasing size $D$. (a)~Error in the global $\rm{L}^1$ norm, largely independent of $D$. (b)~Error in the pointwise $\rm{L}^\infty$ norm, showing an exponential decay of Gibbs oscillation amplitudes due to diminishing jump heights. (c) Median error, deteriorating with increasing $D$ as the measure of the smooth region shrinks.}
    \label{fig:poisson_dense_convergence}
\end{figure}

\begin{figure}[t]
    \centering
    \includegraphics{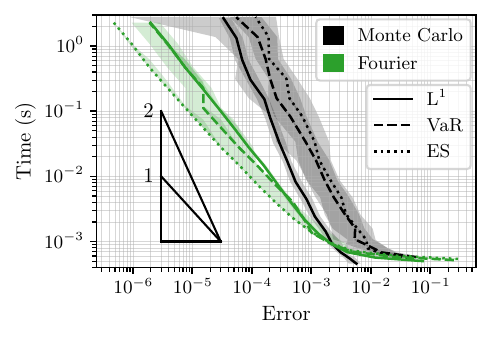}
    \caption{Performance comparison between spectral and Monte Carlo methods for the reconstruction of the WPB CDF (in $\rm{L}^1$ norm) and computation of $99\%$ VaR and ES. Exact reference solutions are obtained via recursive convolution (see Appendix~\ref{appendix:conventional-recursive}). A convergence triangle is shown,  illustrating the expected quadratic advantage of spectral methods.}
    \label{fig:poisson_dense_performance}
\end{figure}

We next consider the weighted Poisson--binomial model for larger instances with an increasing number of components $D$. We restrict to moderate sizes ($2 \leq D \leq 20$) in order to compute exact reference solutions using the recursive convolution algorithm (see Appendix~\ref{appendix:conventional-recursive}). Despite the exponential growth of the support ($M = 2^D$), the global $\rm{L}^1$ error remains essentially independent of $D$ (Fig.~\ref{fig:poisson_dense_convergence}a). Although the number of discontinuities grows exponentially with $D$, the amplitude of each jump decays exponentially due to normalization of the CDF on $[0, 1]$ (Fig.~\ref{fig:poisson_dense_convergence}b), so their cumulative contribution to the global error remains bounded. By contrast, the accuracy at smooth regions deteriorates with increasing $D$, since their effective measure shrinks exponentially (Fig.~\ref{fig:poisson_dense_convergence}c). Consequently, for sufficiently large $D$, pointwise accuracy becomes dominated by the global $\rm{L}^1$ error.

Finally, spectral methods exhibit the expected $\O{\varepsilon^{-1}}$ complexity, yielding a clear quadratic advantage over the $\O{\varepsilon^{-2}}$ complexity of Monte Carlo sampling for both CDF reconstruction and computation of tail-risk metrics (Fig.~\ref{fig:poisson_dense_performance}). This is in full agreement with the theoretical convergence rates derived in Sect.~\ref{sec:qtt-complexity}.

\subsubsection{QTT compression in the large-$D$ regime}

\begin{figure}[t]
    \centering
    \includegraphics{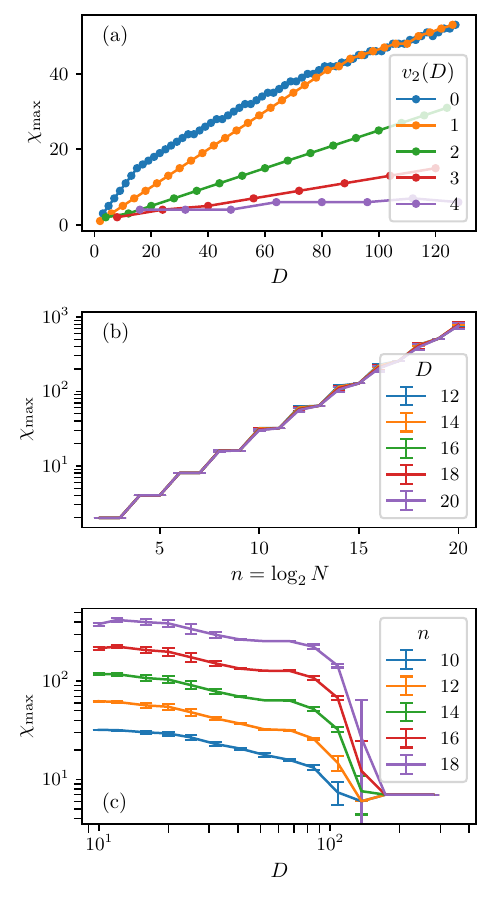}
    \caption{
        Compressibility of the QTT characteristic function for discrete models.
        (a)~Binomial model with $n=20$ cores, determined by the 2-adic valuation $v_2(D)$.
        (b)~WPB model averaged over 10 random instances, showing near-maximal rank growth at small $D$.
        (c)~WPB model, showing a sharp bond-dimension collapse with increasing $D$ and entering a compressible low-rank regime for $D \gtrsim 300$.
    }
    \label{fig:poisson_qtt_cf_maxbond}
\end{figure}

\begin{figure}[t]
    \centering
    \includegraphics{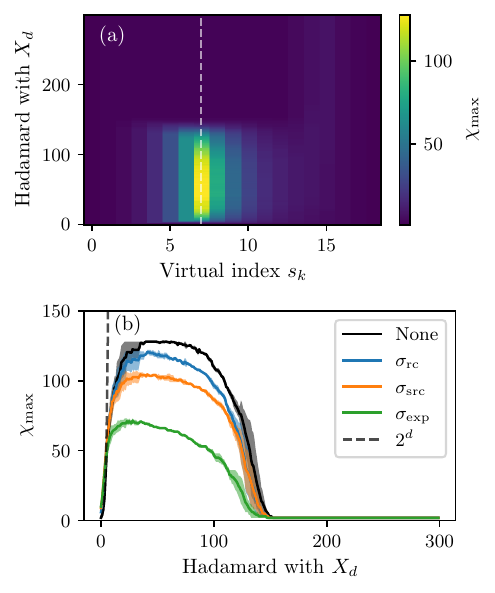}
    \caption{Compressibility of the WPB characteristic function~\eqref{eq:phi_X_multinomial} with $n=18$ and $D=300$. (a) Heat map showing the maximum bond dimension at each virtual index $s_k$ after incorporating each variable $X_d$. (b) Slice of the previous heat map at $s_7$ averaged over 10 independent random instances, showing an initial filter-dependent exponential barrier followed by strong compression.}
    \label{fig:poisson_mps_cf_maxbond_2D}
\end{figure}

\begin{figure*}[t]
    \centering
    \includegraphics[width=\textwidth]{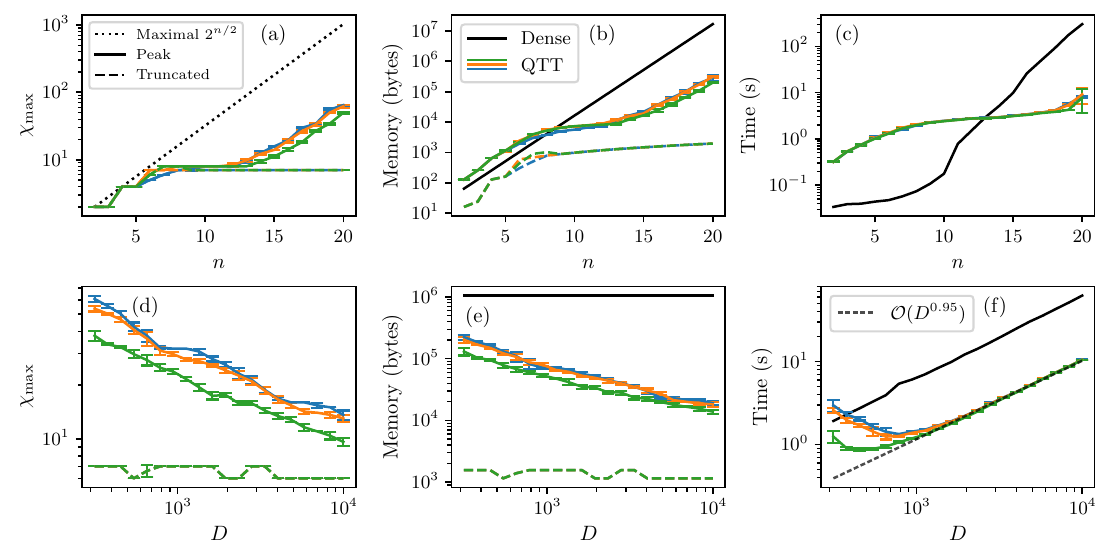}
    \caption{Performance comparison between QTT and dense-vector spectral methods for the computation of the filtered characteristic function. Top row (fixed $D=3000$, varying $n$): (a) peak and truncated bond dimensions, well below the maximal $2^{n/2}$ exponential value; (b) memory usage in bytes; (c) wall-clock time. Bottom row [(d)-(f)] show the same metrics at fixed $n=16$ while varying $D$.}
    \label{fig:poisson_qtt_cf_performance}
\end{figure*}

\begin{figure*}[t]
    \centering
    \includegraphics[width=\textwidth]{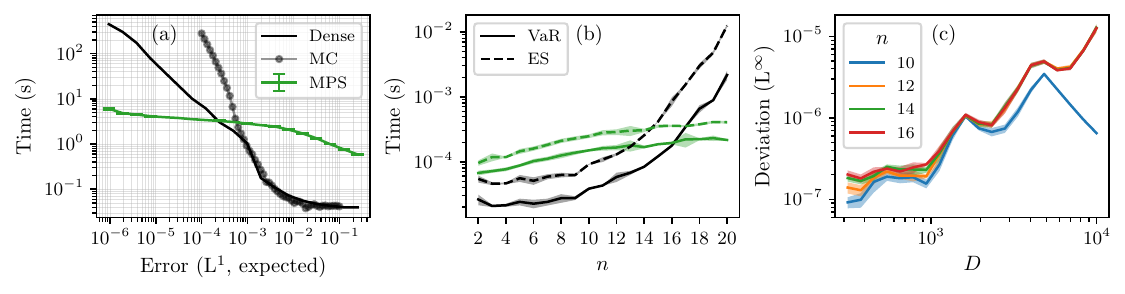}
    \caption{Performance comparison between QTT and dense-vector spectral methods for the computation of the CDF, VaR, and ES.
    (a) Wall-clock time required to reach a prescribed $\rm{L}^1$ tolerance for Monte Carlo, dense (black) and QTT (green) spectral methods.
    (b) Wall-clock time required to compute VaR and ES given a CDF in dense-vector or QTT format.
    (c) Maximum pointwise deviation ($\rm{L}^\infty$) between the dense-vector and QTT representations of the model's CDF.
    }
    \label{fig:poisson_qtt_cdf_performance}
\end{figure*}

We now turn to much larger systems, where the recursive convolution algorithm becomes intractable and where the QTT spectral representation achieves computational advantages over dense vectors. We study the maximal bond dimension of the characteristic function, being the key quantity governing computational cost (see Sect.~\ref{sec:qtt-complexity}).

Structured models can exhibit strong QTT compressibility, determined by the commensurability of the model's support with the frequency samples. We illustrate this for the binomial model, whose commensurability is controlled by the parity structure of $D$, conveniently summarized by the 2-adic valuation $v_2(D)$ (i.e., the number of factors of 2 dividing $D$). For $n=20$ QTT cores, the maximal bond dimension of the binomial characteristic function forms well-separated complexity classes for different values of $v_2(D)$ (Fig.~\ref{fig:poisson_qtt_cf_maxbond}a).

However, for small and moderate $D$, unstructured discrete models are highly adversarial and effectively incompressible. In this regime, the characteristic function becomes a quasiperiodic superposition of incommensurate phases, destroying scale separability and yielding near-maximal bond dimensions. This behavior is illustrated in Fig.~\ref{fig:poisson_qtt_cf_maxbond}b, where bond dimensions grow exponentially with resolution. In this regime, TT-SVD truncation is not effective, and the QTT size approaches that of dense vectors.

A qualitative transition occurs once $D$ is sufficiently large, driven by the multiplicative suppression of high-frequency modes. As shown in Appendix~\ref{appendix:frequency_discretization}, the modulus of the global characteristic function is bounded by a smooth Gaussian envelope centered at zero frequency. Up to frequencies of order $\O{D}$, this envelope dominates, effectively concentrating the spectral weight. Beyond this scale, the incommensurate phases of the local contributions begin to decorrelate. Although these quasiperiodic modulations could produce revivals with $\O{1}$ amplitude, their expected amplitude typically decreases exponentially in $D$, falling well below the TT-SVD truncation tolerance $\epsilon$. This CLT-driven multiplicative suppression concentrates spectral energy within a smooth envelope, enabling strong QTT compression while preserving non-Gaussian features. Consistently, Fig.~\ref{fig:poisson_qtt_cf_maxbond}c shows a sharp bond-dimension collapse at $D$ of a few hundred, beyond which the representation enters a resolution-independent low-rank regime.

This emergent compressibility translates directly into favorable computational performance. In this compressible large-$D$ regime, for $D \gtrsim 300$, both memory usage and runtime scale polylogarithmically with the number of frequency modes, yielding orders-of-magnitude savings relative to dense spectral methods. We make this explicit by focusing on a larger system with $D=3000$ variables, where the computational advantages of the QTT framework are most pronounced. For smaller $D$, a similar qualitative behavior is observed, although the performance gap is reduced.

Dense-vector and QTT-based CF representations are compared as a function of the frequency modes $n = \log_2 N$ in Fig.~\ref{fig:poisson_qtt_cf_performance}. For a fixed number of components $D$, peak bond dimensions initially saturate at a low constant up to a certain resolution threshold, beyond which they scale exponentially with $n$---reflecting an intermediate complexity barrier. Ultimately, the final dimensions after truncation collapse to small, resolution-independent values (Fig.~\ref{fig:poisson_qtt_cf_performance}a). This compression yields orders-of-magnitude memory savings---kilobytes versus tens of megabytes at $n=20$, despite the exponential growth of peak intermediate memory (Fig.~\ref{fig:poisson_qtt_cf_performance}b)---and translates directly into a polynomial wall-clock time scaling, in contrast to the exponential scaling of dense-vector approaches (Fig.~\ref{fig:poisson_qtt_cf_performance}c).

Interestingly, increasing the number of components $D$ enhances intermediate compressibility, reducing peak bond dimensions and associated memory and runtime costs until a linear saturation is reached. Consequently, peak bond dimension scaling does not pose a limitation for large-$D$ scalability. This behavior arises from a smoothing of the local characteristic functions, which are evaluated over an effective frequency window of size $\O{1/D}$. This mechanism is invariant to weight normalization, since the relevant argument depends only on the product $w_d \omega$. In the normalized case, the weights scale as $w_d = \O{1/D}$ with fixed cutoff $\Omega$, while in the unnormalized case the support length grows as $L = \O{D}$, inducing a contraction of the frequency range $\Omega = \O{1/D}$. Hence, both parametrizations sample the same effective frequency region. Since characteristic functions satisfy $\phi(0) = 1$ and remain smooth near the origin, this enforced locality yields increasingly low-rank initial tensors and suppresses bond dimension growth during early Hadamard products. 

We next benchmark computational cost as a function of target accuracy. Since exact reference distributions are unavailable at large $D$, errors are inferred from algorithmic convergence rates. Pointwise deviations between the QTT and dense CDF representations confirm that computational gains do not come at the expense of accuracy (Fig.~\ref{fig:poisson_qtt_cdf_performance}c). Wall-clock times to reach a prescribed global $\rm{L}^1$ tolerance for dense, QTT-based, and Monte Carlo methods are compared in Fig.~\ref{fig:poisson_qtt_cdf_performance}a.

At moderate accuracies, dense spectral methods outperform Monte Carlo sampling. As the tolerance $\varepsilon$ is tightened, a crossover occurs at $\varepsilon \approx 10^{-3}$, below which QTT-based spectral methods become more efficient. In this regime, compressibility makes the effective cost scale polylogarithmically with the number of frequency modes, yielding an up-to exponential runtime advantage over dense implementations.

Given a compressed QTT representation of the CDF, the exponential memory savings translate into substantial computational advantages in post-processing: the wall-clock time to compute VaR and ES scales polylogarithmically with resolution (Fig.~\ref{fig:poisson_qtt_cdf_performance}b; see also Sect.~\ref{sec:qtt-postprocessing}).

The weighted Poisson--binomial model considered here serves as a minimal representative of more general discrete models. In general, increasing the local cardinalities $K_d$ enlarges the support to $\prod_d K_d$, maintaining an exponential growth in $D$ with a larger effective base. Spectrally, this manifests as a more aggressive multiplicative suppression of high-frequency quasiperiodic modulations, allowing the smooth spectral envelope to dominate at significantly smaller system sizes. Consequently, the aggregate bond dimensions are expected to saturate after fewer local Hadamard products, ensuring that the performance advantages observed here extend broadly to more general sums of discrete random variables.

\subsection{Continuous models}
\label{sec:results-continuous}

We now consider continuous models, focusing on weighted sums of independent lognormal random variables. A lognormal random variable $Y$ is governed by the density
\begin{equation}
    f_{Y}(y) = \frac{1}{y \sigma \sqrt{2 \pi}} \exp\left( -\frac{(\log y - \mu)^2}{2 \sigma^2} \right), \quad y > 0,
    \label{eq:f_X_d_lognormal}
\end{equation}
which is characterized by a heavy, sub-exponential right tail.

Approximating the distribution of sums of lognormal variables poses a classical, well-known computational challenge~\cite{fenton1960,dufresne2004}. Unlike Gaussian or Gamma variables, lognormal variables are not closed under the convolution in Eq.~\eqref{eq:f_X_convolution}, and the density of their sum admits no closed-form analytical expression. Moreover, their effective support grows exponentially with the variance of the independent components.

A further complication arises, since the lognormal distribution does not admit a closed-form expression for its characteristic function, which can only be written in an integral form
\begin{equation}
    \phi_Y(\omega) = \frac{1}{\sigma \sqrt{2\pi}}\int_{0}^{\infty} e^{i \omega y} \exp \left( -\frac{(\ln y - \mu)^2}{2 \sigma^2} \right) \frac{\mathrm{d}y}{y},
    \label{eq:phi_Y_lognormal}
\end{equation}
making its computation costlier. Moreover, evaluating Eq.~\eqref{eq:phi_Y_lognormal} with classical quadrature methods is numerically ill-posed due to the highly oscillatory factor $e^{i \omega y}$.

This instability can be tackled by deforming the integration contour in the complex plane. Following Gubner's approach~\cite{gubner2006}, Eq.~\eqref{eq:phi_Y_lognormal} can be expressed as
\begin{equation}
    \phi_Y(\omega) = \frac{e^{\pi^2/8 \sigma^2}}{\sqrt{\pi}} \int_{-\infty}^{\infty}\!\! e^{-z^2} \exp\left( -\left[ \frac{i \pi}{\sqrt{2}\sigma} z + \omega e^{\sqrt{2}\sigma z} \right] \right)  \mathrm{d}z,
    \label{eq:phi_Y_lognormal_gubner}
\end{equation}
where the Gaussian weight $e^{-z^2}$ regularizes the oscillatory behavior of the integrand. Although this expression assumes $\mu=0$, the mean is reintroduced by evaluating the function at a scaled frequency $\omega' = \omega e^\mu$. This representation is now well-suited for Gauss--Hermite quadrature, enabling stable evaluation of the characteristic function at prescribed frequencies.

\subsection{Single-variable analysis}
\label{sec:results-continuous-single}

We approximate the lognormal density in Eq.~\eqref{eq:f_X_d_lognormal} with the Fourier spectral method, illustrating the use of Gubner's formula~\eqref{eq:phi_Y_lognormal_gubner} and continuous QTT representations.

To encode the characteristic function in QTT format, one may resort to TCI as a general-purpose approach. However, TCI relies on heuristic adaptive sampling and can struggle with functions that are highly oscillatory or sharply peaked. This is exactly the case for Gubner's characteristic function, which is sharply peaked and oscillatory near $\omega=0$ but decays smoothly for larger $|\omega|$, becoming prone to convergence issues on deep QTT grids.

Fortunately, Gubner's representation allows a more direct and stable construction. By discretizing Eq.~\eqref{eq:phi_Y_lognormal_gubner} using $K$ Gauss--Hermite quadrature nodes, the characteristic function reduces to a finite weighted sum of $K$ exponential terms in $\omega$---typically, choosing $K=45$ nodes is sufficient for an accurate approximation~\cite{gubner2006}. Similarly to the discrete case in Eq.~\eqref{eq:phi_X_multinomial}, this sum admits an exact QTT representation with bond dimension $\chi=K$, which can be subsequently compressed via TT-SVD. Exploiting this explicit algebraic construction avoids the instabilities that may arise when using TCI for large values of $n$.

The convergence of the spectral reconstruction now depends on the interaction between the distribution parameters $(\mu,\sigma)$, the truncation window $b$, and the number of frequency samples $N$. If $b$ is too small, truncation of the right tail introduces wrap-around artifacts in the Fourier reconstruction, since the density does not match at the boundaries of the truncated domain. We therefore choose $b(\mu,\sigma)$ such that the probability mass beyond $b$ is below a prescribed tolerance $\delta$. Using the lognormal CDF gives
\begin{equation}
    b(\delta) = \exp \left(\mu + \sigma F_G^{-1}(1-\delta)\right),
    \label{eq:lognormal_b}
\end{equation}
where $F_G$ is the standard normal CDF. With this choice, the tail truncation is controlled by $\delta$, while the remaining approximation error depends on $N$.

The required resolution $N$ must be sufficient to resolve the spectrum, as derived rigorously in Section~\ref{sec:qtt-complexity}. However, since the characteristic function has no closed-form expression and its decay cannot be analyzed directly, we use a heuristic estimate based on the spatial resolution of the density. The sharpest feature occurs near the mode $y':=e^{\mu - \sigma^2}$. Requiring the grid spacing to resolve this feature requires the number of samples to scale as
\begin{equation}
    N = \O{\frac{b}{y'}} = \O{\exp \left[\sigma F_G^{-1}(1-\delta) + \sigma^2\right]},
\end{equation}
i.e., exponentially with $\sigma^2$. Hence, accurately resolving high-volatility lognormal densities requires prohibitively large computational resources for dense FFT-based solvers, whereas QTT representations provide a scalable alternative.

\begin{figure}[t]
    \centering
    \includegraphics{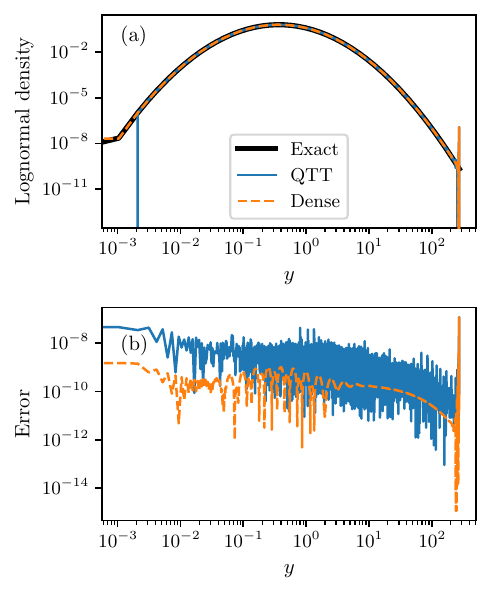}
    \caption{Spectral reconstruction of a single lognormal density in Eq.~\eqref{eq:f_X_d_lognormal} with $\mu=0$, $\sigma=1$, and tail tolerance $\delta=10^{-8}$ ($b \approx 273$). (a) Density function, reconstructed using dense-vector and QTT spectral methods. (b) Reconstruction error for both representations.}
    \label{fig:lognormal_density}
\end{figure}
As an illustration, the spectral reconstructions of a lognormal density with $\mu=0$, $\sigma=1$, and tail tolerance $\delta=10^{-8}$ are shown for both dense and QTT methods in Fig.~\ref{fig:lognormal_density}a, with the corresponding reconstruction errors in Fig.~\ref{fig:lognormal_density}b.

\begin{figure*}[t]
    \centering
    \includegraphics[width=\textwidth]{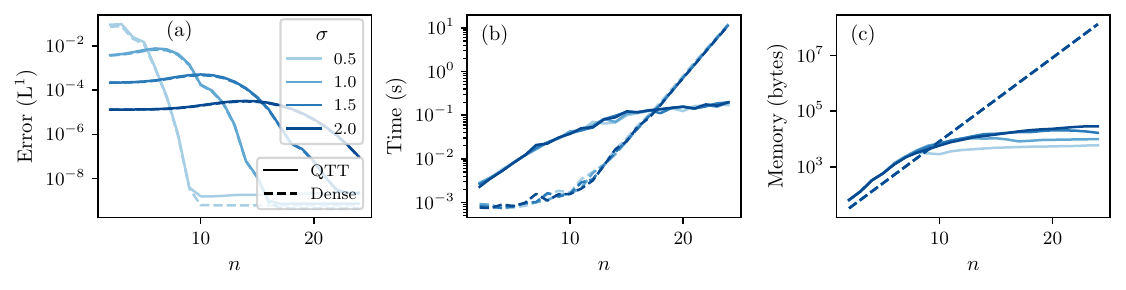}
    \caption{Performance comparison between QTT and dense-vector spectral methods for the computation of a single lognormal density function. (a) $\mathrm{L}^1$ reconstruction error for several distinct variances shown in the legend. (b) Wall-clock time for increasing resolutions $n = \log_2 N$. (c) Memory usage (in bytes).}
    \label{fig:lognormal_single_variable}
\end{figure*}
We next benchmark dense and QTT representations for several values of the variance. The $\mathrm{L}^1$ errors closely match across methods (Fig.~\ref{fig:lognormal_single_variable}a), confirming that the analytical QTT encoding accurately reconstructs the density from Eq.~\eqref{eq:phi_Y_lognormal_gubner}. The computational costs differ substantially: the bond dimensions remain bounded for all values of $\sigma$ (not shown), so QTT representations exhibit a clear polylogarithmic scaling with $N$ in both wall-clock time and memory (Figs.~\ref{fig:lognormal_single_variable}b,c).

\subsection{Weighted-sum analysis}
\label{sec:results-continuous-sum}

We now consider weighted sums of independent lognormal variables, defined as
\begin{equation}
    X_\rm{Log} = \sum_{d=1}^D w_d Y_d.
    \label{eq:X_LS}
\end{equation}

Although $X_\rm{Log}$ no longer has a closed-form expression for its density, several methods exist to estimate its support and determine the truncated support size $b$. A common approach is the Fenton--Wilkinson approximation, which matches the first two moments to model $X_\rm{Log}$ as a proxy lognormal distribution~\cite{fenton1960}. However, this method severely underestimates the extreme right tail, becoming unreliable when high precision is required.

Instead, we exploit the fact that the lognormal distribution is subexponential, so extreme sum values are typically driven by the single largest component rather than the collective contribution of multiple variables. Consequently, the tail of their sum is asymptotically equivalent to the tail of the largest component~\cite{asmussen2008}. In particular, for a sufficiently large threshold $b$, the survival function of the weighted sum~\eqref{eq:X_LS} is well approximated by
\begin{equation}
    \Pr(X_\rm{Log} > b) \approx \sum_{d=1}^D \Pr(w_d Y_d > b).
\end{equation}
Since $\log(w_d Y_d) \sim \cal{N}(\mu_d + \log w_d, \ \sigma_d^2)$, we enforce a prescribed tail tolerance $\delta$ by determining the boundary $b$ as the root of the strictly monotone equation
\begin{equation}
    \sum_{d=1}^D \left[ 1 - F_G\left(\frac{\ln b - \mu_d - \ln w_d}{\sigma_d}\right) \right] = \delta,
\end{equation}
which can be solved by a bisection method. This relation becomes accurate for sufficiently small values of $\delta$ and depends on the size of the component variances $\sigma_d$, since larger variances promote the dominance of a single extreme component.

Because the true density of $X_\rm{Log}$ is not available, we assess convergence through a self-convergence criterion. Let $f^{n}(y)$ denote the reconstructed density on a grid with $N=2^n$ points. Fixing the domain $[0, b]$ and increasing $n$, we define the self-error as the difference between successive resolutions using the discrete $\rm{L}^2$ norm, $\|f^{n} - f^{n-1}\|_{\rm{L}^2}$. To ensure consistent dimensionality, the fine-grid reconstruction $f^{n}$ is restricted to the coarse-grid nodes of $f^{n-1}$ by subsampling the even nodes of $f$. We assume the $\rm{L}^2$ norm since it enables an efficient computation in QTT format through tensor contractions; however, convergence in $\rm{L}^2$ implies convergence in $\rm{L}^1$ via the Cauchy--Schwarz inequality, so monitoring the decay of the self-error provides an adaptive convergence criterion for the spectral reconstruction.

\begin{figure*}[t]
    \centering
    \includegraphics[width=\linewidth]{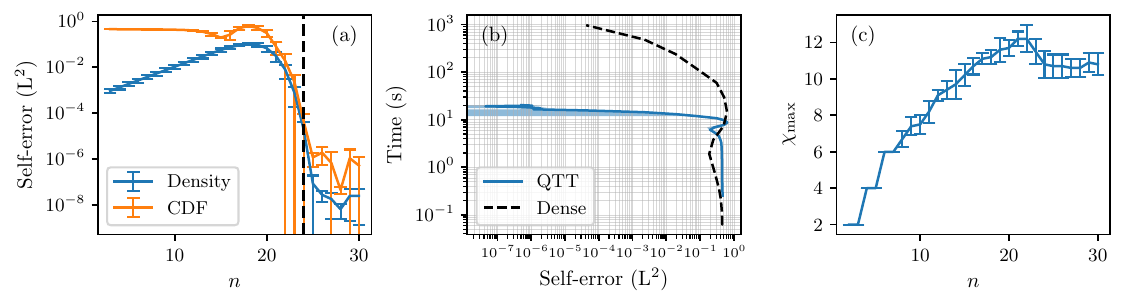}
    \caption{Performance comparison between QTT and dense-vectors for the evaluation of a weighted sum of $D=100$ lognormal variables. The results are averaged across $10$ random instances.
    (a)~Self-error for the density $f^n$ (absolute $\mathrm{L}^2$ norm) and CDF $F^n$ (relative) for increasing resolutions. Dense and QTT results match exactly, so the dense curves are omitted for clarity. The vertical dashed line indicates the largest dense resolution achieved ($n=24$).
    (b)~Wall-clock time versus CDF self-error, comparing dense (black, dashed) and QTT (blue, solid) implementations.
    (c)~Maximum bond dimension of the QTT CDF, saturating at $n \approx 20$ and indicating strong compressibility.
    }
    \label{fig:lognormal_weighted_sum}
\end{figure*}
Numerical results for a sum of $D=100$ independent lognormal components are reported in Fig.~\ref{fig:lognormal_weighted_sum}. The tail tolerance is fixed at $\delta = 10^{-10}$ to ensure the domain $b$ encompasses the deep right tail. To test robustness, we consider 10 independent instances with parameters sampled from $\mu_d \sim \cal{U}(-1.0, 1.0)$, $\sigma_d \sim \cal{U}(1.0, 3.0)$, and weights $w_d \sim \cal{U}(0.0, 1.0)$ normalized to unity. 

The self-convergence of the density and CDF follows two phases (Fig.~\ref{fig:lognormal_weighted_sum}a): for low $n$, the error initially increases as the grid begins to resolve the features of the density mode $y'$. Once this resolution threshold is surpassed, the error monotonically plunges to the TT-SVD truncation floor $\epsilon = 10^{-8}$. 

While the dense and QTT-based reconstructions exhibit the same convergence rate, their performance differs significantly. Dense vectors become computationally impractical beyond $n=24$ due to the exponential growth of the memory requirements. By contrast, the QTT format effortlessly handles discretizations up to $n=30$---a billion frequency modes---and beyond on standard hardware, entirely circumventing the memory constraints of dense-vector approaches.

Besides enabling otherwise intractable resolutions, the QTT approach is substantially more efficient. The maximal bond dimension saturates at a low constant as $n$ increases (Fig.~\ref{fig:lognormal_weighted_sum}c), leading to massive memory compression reflected in the efficiency benchmark (Fig.~\ref{fig:lognormal_weighted_sum}b): the QTT approach follows a polylogarithmic scaling in $N$, providing up to exponential speedups in time and memory over dense-vector methods and several orders of magnitude higher precision before reaching the hardware's memory ceiling.

\section{Conclusion}
\label{sec:conclusion}

The central finding of this work is that characteristic functions of weighted-sum models exhibit low-rank structure in the quantics tensor train (QTT) representation. This structure arises from spectral smoothness or CLT-driven spectral concentration, independently of whether the underlying distribution is Gaussian. By exploiting this compressibility within a Fourier spectral framework, we achieve up-to exponential computational advantages over dense methods for computing the fully non-Gaussian target distribution, its CDF, and derived risk measures such as Value at Risk and Expected Shortfall.

The framework applies uniformly to discrete and continuous models through complementary compressibility mechanisms. In discrete models, multiplicative suppression of high-frequency quasiperiodic modulations concentrates spectral energy into a low-rank envelope as the number of components grows. This is a CLT-driven effect that does not require the target distribution to be near Gaussian. In continuous models, the intrinsic spectral smoothness of the local characteristic functions yields compressibility directly. Both mechanisms were validated on weighted sums of Bernoulli and lognormal variables, achieving orders-of-magnitude memory compression and polylogarithmic runtime scaling.

Several directions remain open. Identifying broader classes of probabilistic models with intrinsically low-rank characteristic functions would clarify the scope of the approach. Extensions to weakly correlated variables and to higher-dimensional settings are two natural next steps, as tensor-network representations may offer one of the few scalable strategies in that regime. Parallel implementations using binary-tree reduction can further exploit the embarrassingly parallel structure of the Hadamard product sequence. More broadly, this work establishes that fully non-Gaussian probability distributions of widespread weighted-sum models are computationally accessible at scale, through a tensor-network framework whose efficiency is governed by spectral compressibility rather than proximity to Gaussianity.

\section*{Code and data availability}

The numerical experiments and benchmarks reported in this work were performed on a workstation equipped with an 18-core Intel(R) Xeon(R) W-2295 CPU and 64GB of RAM.

The data and source code required to reproduce the numerical results and figures presented in this manuscript are openly available on GitHub (see Ref.~\cite{aldavero2026_github_repo}). The implementation relies on the open-source tensor network library SeeMPS~\cite{seemps2026}.

\section*{Acknowledgments}
This work was supported by the EuroHPC Joint Undertaking under project QEC4QEA (Ref. 101194322), and by project PCI2025-163266 funded by MICIU/AEI/10.13039/501100011033. The authors also acknowledge support from CSIC'S PRO-ERC AGAIN 2025 funding. JJGR acknowledges Andrea Cadarso for insightful discussions on the problem formulation and its applications in finance.

\appendix

\section{Conventional numerical methods}
\label{appendix:conventional}

We now briefly review conventional numerical methods for evaluating the CDF of weighted-sum models defined in Eq.~\eqref{eq:X}, focusing on Monte Carlo sampling and recursive convolution methods.

\subsection{Monte Carlo sampling}
\label{appendix:conventional-montecarlo}

Monte Carlo sampling is a widespread method for evaluating generic probabilistic models by drawing i.i.d. samples from their random variable $X$~\cite{robert1999}. Given $S$ samples $\{X^{(s)}\}_{s=1}^S$, the CDF is approximated by the empirical estimator
\begin{equation}
    \widehat{F}_X(x) := \frac{1}{S}\sum_{s=1}^S \Theta(x - X^{(s)}),
    \label{eq:mc_estimator}
\end{equation}
where $\Theta$ denotes the Heaviside step function, acting as an indicator function. By construction, the estimator has a variance
\begin{equation}
    \rm{Var}\left[\widehat{F}_X(x)\right] = \frac{F_X(x)\left(1-F_X(x)\right)}{S}.
    \label{eq:mc_variance}
\end{equation}
Consequently, the standard deviation decays at the dimension-independent rate $\mathcal{O}(S^{-1/2})$. However, the relative error, defined as
\begin{equation}
    \varepsilon_{\rm{rel}} := \frac{\sqrt{\rm{Var}[\widehat F_X(x)]}}{F_X(x)}
    \label{eq:mc_relative_error}
\end{equation}
decays as $[S F_X(x)]^{-1/2}$ in the left tail and $[S \ (1-F_X(x))]^{-1/2}$ in the right tail. As a result, accurate estimation of rare-event probabilities and tail-dependent quantities (where $F(x)$ becomes vanishingly small) requires a prohibitively large number of samples, which scales as $S = \O{F_X^{-1} \varepsilon_{\rm{rel}}^{-2}}$. This limitation is particularly severe for heavy-tailed or slowly decaying distributions, where the CDF is nearly flat over wide regions and small sampling fluctuations translate into large uncertainties in quantile location.

\subsection{Recursive convolution methods}
\label{appendix:conventional-recursive}

In the discrete case, sums of categorical variables $X_d$~\eqref{eq:X_d_categorical} can be evaluated exactly. A naive combinatorial approach, which considers all possible configurations of the model, scales exponentially with the number of components $D$, and is therefore restricted to very small problem sizes. 

An efficient alternative is the recursive convolution method, which incorporates the components recursively to update the probability mass function (PMF) of the running partial sum~\cite{panjer1981, wang1993}. Let
\begin{equation}
    f^{(\ell)}(x_m) := \Pr \left(\sum_{d=1}^{\ell} w_d X_d = x_m\right),
    \label{eq:f_ell}
\end{equation}
denote the PMF of the model after incorporating the first $\ell$ components, initialized as $f^{(0)}(0)=1$ and $f^{(0)}(x_m)=0$ for $x_m\neq 0$. Conditioning on the $(\ell+1)$-th component yields the recursion
\begin{equation}
    f^{(\ell + 1)}(x_m) = \sum_{k=1}^{K_{\ell + 1}} p_{k (\ell + 1)} f^{(\ell)}\left(x_m - w_{\ell + 1} x_{k (\ell+1)}\right),
    \label{eq:f_ell_recursion}
\end{equation}
where $x_{k(\ell+1)}$ and $p_{k(\ell+1)}$ denote respectively the value and probability of the $k$-th outcome of the categorical variable $X_{\ell+1}$, efficiently updating the PMF by shifting and weighting the current distribution.

Its computational cost is governed by the size of the intermediate supports. Let $M_\ell$ denote the size of the support of $f^{(\ell)}$. A direct implementation of Eq.~\eqref{eq:f_ell_recursion} requires $\O{\sum_{\ell=0}^{D-1} K_{\ell+1} M_\ell}$ operations and $\O{M_D}$ memory. In the absence of degeneracies, the final support size may grow as $\prod_{d=1}^D K_d$, leading to exponential complexity. Consequently, recursive convolution methods are practical only for highly structured models, remaining impractical for generic models with incommensurate weights.

\section{Frequency discretization}
\label{appendix:frequency_discretization}

In numerical implementations, the continuous frequency variable $\omega$ is discretized on a finite uniform grid of $N$ modes, given by
\begin{equation}
    \omega_k  = -\Omega + k \Delta \omega, \quad \Delta \omega = \frac{2 \Omega}{N}, \quad k = 0, 1, \ldots, N-1,
    \label{eq:omega_k}
\end{equation}
with cutoff $\Omega$. The parameters $\Omega$ and $N$ determine the accuracy of the spectral reconstruction, which are constrained by classical sampling relations.

\paragraph{Continuous densities.}
In continuous scenarios, these parameters are related through the Shannon--Nyquist sampling theorem, which we state without proof~\cite{shannon2006}.
\begin{theorem}{Shannon--Nyquist sampling theorem}\\
    Let $f(x)$ be a real-valued function whose Fourier transform $\phi(\omega)$ is supported on $[-\Omega,\Omega]$. Then $f$ is uniquely determined by its samples $f(x_k)$ at $x_k = k \Delta x$, $k\in\mathbb Z$, provided
    \[
    \Delta x \le \frac{\pi}{\Omega}.
    \]
    \label{thm:shannon-nyquist}
\end{theorem}
Given a distribution defined on an interval of length $L$, with spacing $\Delta x = \frac{L}{N}$, resolving spectral content up to bandwidth $\Omega$ requires
\begin{equation}
    N(\Omega) = 1 + \left\lceil \frac{L \Omega}{\pi} \right\rceil
    \label{eq:N_nyquist}
\end{equation}
frequency modes. When the distribution is not band-limited but presents a decaying spectrum, the approximation error decays monotonically with increasing $\Omega$ as more spectral content is captured.

\paragraph{Discrete models.}
In the discrete case, an analogous constraint arises through periodicity. A PMF defined on a uniform grid with spacing $\Delta x$ presents a periodic spectrum with period $2\pi/\Delta x$. Hence, choosing $\Omega=\frac{\pi}{\Delta x}$ suffices to capture all spectral content, enabling exact reconstruction. This is in direct analogy with the Nyquist condition in Theorem~\ref{thm:shannon-nyquist}.

However, generic discrete models are often defined on non-uniform grids, with no common spacing. The resulting spectrum is no longer periodic, but quasiperiodic, exhibiting non-decaying oscillatory structure. Consequently, no finite cutoff $\Omega$ suffices for exact reconstruction, and truncation produces persistent Gibbs oscillations. This motivates the use of spectral filtering to induce a smooth spectrum decay and control high-frequency artifacts.

As an illustrative example, consider the characteristic function of the weighted Poisson--binomial model~\eqref{eq:phi_X_multinomial}. Its modulus exhibits a quasiperiodic structure
\begin{equation}
    |\phi_{\rm{WPB}}(\omega)| = \prod_{d=1}^D 
        \sqrt{p_d^2+(1-p_d)^2+2p_d(1-p_d)\cos(\omega w_d)}.
    \label{eq:phi_X_wpb_envelope}
\end{equation}
Near $\omega=0$, this magnitude is bounded by a Gaussian envelope parameterized by the model's variance,
\begin{equation}
    |\phi_{\rm{WPB}}(\omega)| \ \le \ \exp\left(-\tfrac{2\sigma_{\rm{WPB}}^2}{\pi^2}\omega^2\right),
    \qquad |\omega|\le \tfrac{\pi}{w_{\max}}.
    \label{eq:phi_X_wpb_envelope_lobe}
\end{equation}
The extent of this envelope relative to $\Omega$ scales as $\O{D}$, independent of the weight normalization and the specific values of the weights. Beyond this envelope, for $|\omega| > \frac{\pi}{w_{\max}}$, the incommensurate cosine terms generate quasiperiodic modulations that mathematically prevent exact reconstruction at any finite $\Omega$.

However, a key mechanism governs the QTT compressibility of these discrete models. Although these modulations could in principle reach $\O{1}$ amplitudes, this requires increasingly unlikely phase realignments. In fully incommensurate cases, the phases rapidly decorrelate, and since $ 0 \leq |\phi_d| \leq 1$, the product of $D$ independent, out-of-phase terms drives their typical amplitude to decay exponentially in $D$. This multiplicative suppression pushes high-frequency contributions below the truncation tolerance, allowing tensor-network methods to discard them and isolate the low-rank central envelope.

\section{Spectral filtering convergence}
\label{appendix:filtering_convergence}

Spectral filtering provides rigorous convergence guarantees for Fourier reconstructions of piecewise-continuous functions. A fundamental result, due to Gottlieb and Shu, quantifies their accuracy sufficiently away from the discontinuities~\cite{gottlieb1997}. Let us state a simplified version without proof.
\begin{theorem}{Spectral filtering convergence theorem} \\
    Let $F(x)$ be a piecewise $C^q$ function on a periodic domain of length $L$, with a single jump discontinuity at $\xi$. Let $\sigma(\eta)$ be a real spectral filter of order $q$, satisfying:
    \begin{enumerate}
        \item $\sigma(0) = 1$ and $\sigma^{(\ell)}(0)=0$ for $1 \leq \ell \leq q-1$,
        \item $\sigma(\eta) = 0$ for $|\eta| \geq 1$,
        \item $\sigma\in C^{q-1}(\mathbb R)$.
    \end{enumerate}
    Define the filtered partial sum
    \[ 
    F^{N,\sigma}(x) \ = \ \sum_{k=-\nu}^{\nu} \hat F_k \sigma \left(\tfrac{k}{\nu}\right) e^{ik\tfrac{2\pi}{L}x},
    \]
    where $N = 2 \nu + 1$ Fourier modes are used. Let
    \[
    d(x) = \min_{k=-1,0,1} |x - \xi + L k|
    \]
    denote the distance from $x$ to the discontinuity modulo periodicity. Then there exists a constant $A$, independent of $N$, such that
    \[
    |F(x) - F^{N, \sigma}(x)| \leq A \ N^{1-q} \ d(x)^{1-q} \ \|F\|_{H^q},
    \]
    where $\|F\|_{H^q}$ is the local Sobolev $q$-norm over the smooth subintervals.
    \label{thm:spectral_filtering}
\end{theorem}
This estimate extends to functions with finitely many jump discontinuities by applying the bound locally on each smooth subinterval. The distance term $d(x)$ governs the error near the discontinuities, while the Sobolev term $\|F\|_{H^q}$ only contributes a finite, model-dependent prefactor.

This result implies spectral convergence away from discontinuities---where the distance term $d(x)$ is bounded away from zero---determined by the smoothness of the filter. Polynomial filters of order $q$ yield algebraic convergence $\O{N^{1-q}}$, while exponential filters achieve $\O{e^{-N}}$ convergence. However, near the discontinuities, Gibbs oscillations persist, precluding uniform convergence.

For self-containedness, let us derive a well-known global convergence rate in the $\rm{L}^p$ norm for the filtered CDF reconstruction defined in Eq.~\eqref{eq:F_N_sigma}. Following the Nyquist criterion~\eqref{thm:shannon-nyquist}, the parameters $N$, $\Omega$, and $\Delta x$ are linked via Eq.~\eqref{eq:N_nyquist}. Since $N \propto \Omega$, let us denote the pointwise error in terms of the frequency cutoff $\Omega$ as
\begin{equation}
    E_\Omega(x) \equiv F(x) - F^{\Omega, \sigma}(x) = \O{\Omega^{1-q} \ d(x)^{1-q}}.
\end{equation}

\begin{proposition}{Global $\rm{L}^p$ convergence of filtered spectral approximations.}

Let $F(x)$ be defined as in Theorem~\ref{thm:spectral_filtering} with $J$ jump discontinuities, and let $F^{\Omega, \sigma}(x)$ be its filtered spectral reconstruction as in Eq.~\eqref{eq:F_N_sigma} with a filter of order $q$ such that $p(q-1) > 1$. Then, the global error $E_\Omega(x)$ is bounded in the $\rm{L}^p$ norm by
\begin{equation}
    \|E_\Omega\|_{\rm{L}^p} = \O{\Omega^{-1/p}}.
    \label{eq:gibbs-band-rate}
\end{equation}
\label{prop:global_convergence}
\end{proposition}

\begin{proof}
Let $\{\xi_j\}_{j=1}^J$ denote the locations of the $J$ jump discontinuities of $F$. Also, let us define the Gibbs-dominated regions
\begin{equation}
    \cal{N}_\Omega := \{x : d_j(x) < \delta\},  \qquad |\cal{N}_\Omega| = \frac{2 J c}{\Omega},
\end{equation}
with measure $|\cal{N}_\Omega|$, for a distance $\delta = \tfrac{c}{\Omega}$, with $c>0$ chosen such that neighboring regions do not overlap. Given their complement $\cal{R}_\Omega := [0, L] \ \setminus \ \cal{N}_\Omega$, the global error decomposes as
\begin{equation}
    \|E_\Omega\|_{L^p}^p = \int_{\cal{N}_\Omega} |E_\Omega|^p \rm{d}x + \int_{\cal{R}_\Omega} |E_\Omega|^p \rm{d}x.
    \label{eq:filtered_global_error}
\end{equation}
We are now able to bound the global error $\|E_\Omega\|_{\rm{L}^p}^p$ near discontinuities, where the first term dominates, and in smooth regions, dominated by the second integral.

\paragraph{Near discontinuities.} Within each Gibbs-dominated region, the spectral reconstruction exhibits a bounded overshoot proportional to the jump height. Let $c_0>0$ be a constant proportional to the maximum jump magnitude $\max_j |F(\xi_j^+) - F(\xi_j^-)|$. Then, we can bound the first term of Eq.~\eqref{eq:filtered_global_error} as
\begin{equation}
    \int_{\cal{N}_\Omega} |E_\Omega(x)|^p \rm{d}x \leq c_0^p |\cal{N}_\Omega| = \O{\Omega^{-1}}.
\end{equation}

\paragraph{Away from discontinuities.}
We evaluate the second term of Eq.~\eqref{eq:filtered_global_error} by invoking Theorem~\ref{thm:spectral_filtering}. By definition, polynomial filters of order $q \ge 1$ yield
\begin{equation}
    |E_\Omega(x)| \leq c_1 (\Omega \ d(x))^{1-q},
\end{equation}
for some constant $c_1>0$. The integral follows as
\begin{equation}
    \int_{\cal{R}_\Omega} |E_\Omega(x)|^p \rm{d}x \le c_1^p \Omega^{p(1-q)} \int_{\cal{R}_\Omega} d(x)^{p(1-q)}\rm{d}x.
\end{equation}
Because $d(x)$ measures the distance to the nearest discontinuity, evaluating its integral over the smooth region $\cal{R}_\Omega$ is equivalent to integrating the distance $r$ outwards from the boundary of the Gibbs region ($r = \delta$) up to a maximum macroscopic distance $\Delta = \O{1}$. Thus, the latter integral reduces to $\int_\delta^\Delta r^{-p(q-1)}\rm{d}r$, which evaluates to
\begin{equation}
    \int_\delta^\Delta r^{-p(q-1)}\rm{d}r =
    \begin{cases}
        \O{1}, & p(q-1) < 1, \\
        \O{\log \delta^{-1}}, & p(q-1) = 1, \\
        \O{\delta^{1-p(q-1)}}, & p(q-1) > 1.
    \end{cases}
\end{equation}
Since $\delta \sim \Omega^{-1}$, only the highest order case $p(q-1) > 1$ yields an integral of $\cal{O}(\Omega^{p(q-1)-1})$. Multiplied by the $\Omega^{p(1-q)}$ prefactor, this evaluates to
\begin{equation}
    \int_{\cal{R}_\Omega} |E_\Omega(x)|^p \rm{d}x = \O{\Omega^{-1}}.
\end{equation}
Assuming the filter order $q$ is chosen large enough such that $p(q-1) > 1$, the smooth region error scales exactly with the Gibbs band error. For exponential filters, the bound is sharper, but $\mathcal{O}(\Omega^{-1})$ remains the worst-case scaling from the Gibbs regions. 

Combining both contributions, we obtain the filter-independent bound
\begin{equation}
    \|E_\Omega\|_{\rm{L}^p} = \O{\Omega^{-1/p}},
    \label{eq:gibbs-band-rate}
\end{equation}
and, in particular, $\|E_\Omega\|_{\rm{L}^1}=\cal{O}(\Omega^{-1})$. This rate reflects the shrinking total measure of the Gibbs-dominated regions. By contrast, global $\rm{L}^\infty$ convergence is not achieved.
\end{proof}

\bibliographystyle{apsrev4-2}
\bibliography{references}

\end{document}